\documentclass[10pt,twocolumn,letterpaper]{article}

\usepackage[pagenumbers]{cvpr} % To force page numbers, e.g. for an arXiv version

\usepackage{bm}
\usepackage{graphicx}
\usepackage{amsmath}
\usepackage{amssymb}
\usepackage{mathtools}
\usepackage{amsthm}
\usepackage{booktabs}
\usepackage{caption}
\usepackage{subcaption}
\usepackage{multicol}
\usepackage{stfloats}
\usepackage{algorithm}
\usepackage{makecell}
\usepackage{multirow}
\usepackage{pifont}
\usepackage[table]{xcolor}
\usepackage{algpseudocode}

\newtheorem{theorem}{Theorem}

\newtheorem{lemma}[theorem]{Lemma}

\newtheorem{assumption}[theorem]{Assumption}

\definecolor{cvprblue}{rgb}{0.21,0.49,0.74}
\usepackage[pagebackref,breaklinks,colorlinks,allcolors=cvprblue]{hyperref}

\def\paperID{} % *** Enter the Paper ID here
\def\confName{CVPR}
\def\confYear{2026}

\title{The Power of Decaying Steps: Enhancing Attack Stability and Transferability for Sign-based Optimizers}

\author{Wei Tao\textsuperscript{\rm 1, \rm2, $\dagger$},  Yang Dai\textsuperscript{\rm 1, $\dagger$},  Jincai Huang\textsuperscript{\rm 1}, 
Qing Tao\textsuperscript{\rm 3,*}\\
\textsuperscript{\rm 1} Laboratory for Big Data and Decision, National University of Defense and Technology, China\\
\textsuperscript{\rm 2} Academy of Military Science, China \ \ \
\textsuperscript{\rm 3} Hefei Institute of Technology, China \\
{\tt\small \{taowei, daiyang2000, huangjincai\}@nudt.edu.cn, taoqing@gmail.com}
}

\begin{document}
\maketitle

\def\thefootnote{$\dagger$\ }
\begin{NoHyper}
\footnotetext{Equal contribution, \quad *\ {Corresponding author}}
\end{NoHyper}
\def\thefootnote{\arabic{footnote}}

\begin{abstract}
Crafting adversarial examples can be formulated as an optimization problem. While sign-based optimizers such as I-FGSM and MI-FGSM have become the de facto standard for the induced optimization problems, there still exist several unsolved problems in theoretical grounding and practical reliability especially in non-convergence and instability, which inevitably influences their transferability. Contrary to the expectation, we observe that the attack success rate may degrade sharply when more number of iterations are conducted. In this paper, we address these issues from an optimization perspective. By reformulating the sign-based optimizer as a specific coordinate-wise gradient descent, we argue that one cause for non-convergence and instability is their non-decaying step-size scheduling. Based upon this viewpoint, we propose a series of new attack algorithms that enforce Monotonically Decreasing Coordinate-wise Step-sizes (MDCS) within sign-based optimizers. Typically, we further provide theoretical guarantees proving that MDCS-MI attains an optimal convergence rate of $O(1/\sqrt{T})$, where $T$ is the number of iterations. Extensive experiments on image classification and cross-modal retrieval tasks demonstrate that our approach not only significantly improves transferability but also enhances attack stability compared to state-of-the-art sign-based methods. Code is available at \url{https://github.com/AndssY/MDCS_attack}.

\end{abstract}    
\section{Introduction}
\label{sec:intro}

Crafting adversarial examples (AEs) for a given learning model can be cast as a constrained local optimization problem \cite{Goodfellow2015ExplainingAH}. The defining objective is to seek an imperceptible perturbation to the input data that induces a maximal loss in the model's output. In the landscape of attack methodologies, sign-based gradient optimizers have long been the dominant paradigm. This family of algorithms includes the foundational single-step FGSM \cite{Goodfellow2015ExplainingAH}, its multi-step refinements I-FGSM (BIM) \cite{Kurakin2017AdversarialEI} and PGD \cite{Madry2018TowardsDL}, as well as the momentum-enhanced MI-FGSM \cite{Dong2018BoostingAA}.

FGSM and I-FGSM are directly developed from the well-known gradient descent principle, in which the sign of the gradient vector is used as its update direction. Similar to sign-gradient, MI-FGSM \cite{Dong2018BoostingAA} uses the sign of a variant of Polyak's heavy-ball (HB) momentum \cite{Polyak1964SomeMO} as its iterative direction. By accumulating the past gradients in momentum, MI-FGSM can stabilize update directions and then remarkably boost the transferability of AEs.

MI-FGSM has won the first place in NIPS 2017 Non-targeted Adversarial Attack and Targeted Adversarial Attack competitions \cite{Dong2018BoostingAA}. To further enhance transferability, various sign-based momentum attacks have been proposed by incorporating insight from optimization theory. Typical examples include VMI-FGSM \cite{Wang2021BoostingTT}, GRA \cite{Zhu2024GRA}, PGN \cite{Ge2023Boosting}, MEF \cite{Qiu2024Enhancing}, and MUMIDIG \cite{ren2025mumodig}. Concurrently, recent studies have exposed that Vision-Language Models (VLMs) remain vulnerable to AEs \cite{zhang2022towards, cui2024robustness}. A key distinction in this multimodal setting is the necessity for collaborative perturbation across modalities, rather than independent attacks. Pioneering this direction, Co-Attack \cite{zhang2022towards} collectively generate adversarial perturbations for both image and text modalities. Subsequent methods like SGA, DRA, and SA-AET \cite{gao2024dra,lu2023sga,jia2025semantic} have further pushed the state-of-the-art by augmenting image-text pairs. Notably, despite the progressive complexity in their multimodal attack strategies, these methods share a common algorithmic backbone, i.e., they fundamentally rely on sign-based optimizers as the core engine for perturbation generation.

Despite the demonstrated transferability of sign-based optimizers, there still exist several unsolved problems in theoretical grounding and practical reliability. Theoretically, the convergence guarantees of these methods are fundamentally flawed. As established by Karimireddy et al. \cite{Karimireddy2019ErrorFF}, even in simple convex settings which are a canonical test bed for optimization, sign-based gradients provably fail to converge to optimum. This divergence is intrinsic to the sign operator. It only acts as an extreme form of compression, discarding crucial information about gradient magnitudes and distorting the true descent direction. 

The theoretical fragility manifests in practice as a counter-intuitive operational instability. Contrary to the expectation that conducting more iterations would enhance an attack, we observe an obvious degradation in success rates under prolonged optimization. As quantified in Fig. \ref{fig:instable ASR}, the performance of I-FGSM and MI-FGSM deteriorates when more number of iterations are conducted. Notably, the success rate of I-FGSM plummets from a peak of 37.1\% $(t=2)$ to a trough of 15.5\% $(t=20)$. This performance collapse indicates that the sign-based optimizers are prone to overshooting or oscillating around the optimal attack rate, thereby highlighting a critical limitation in their operational stability and reliability.

\begin{figure}[htbp]
\setlength{\abovecaptionskip}{-0.05cm}
\setlength{\belowcaptionskip}{-0.4cm}
\centering
\includegraphics[width=0.41\textwidth]{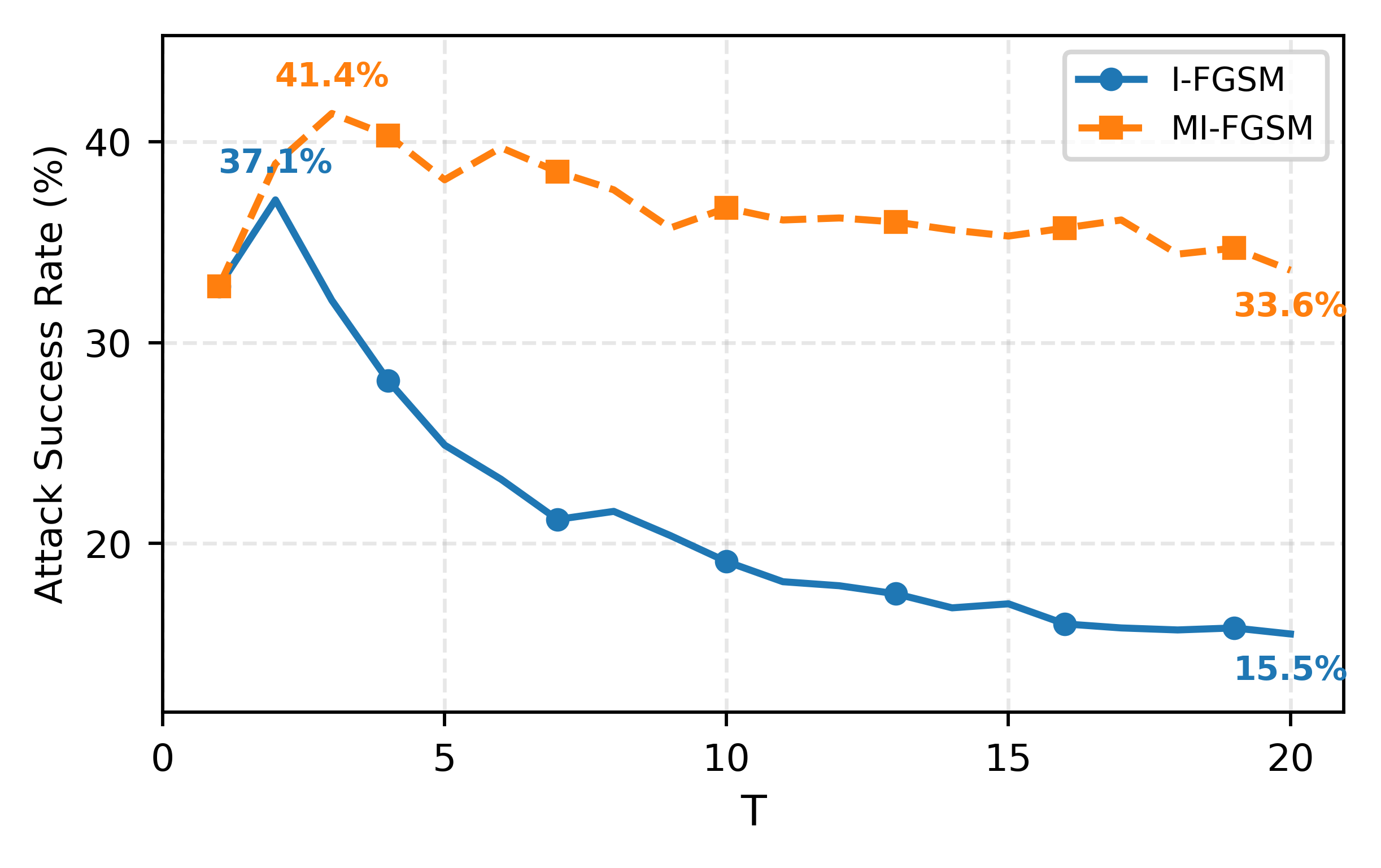}
\caption{Stability comparison of typical sign-based transfer attacks on the NIPS2017 dataset. We employ Res50 as surrogate model and Inc-v3 as target model with $\epsilon = 16/255$.}
\label{fig:instable ASR}
\end{figure}

\begin{figure}[htbp]
\setlength{\abovecaptionskip}{-0.05cm}
\setlength{\belowcaptionskip}{-0.1cm}
\centering
\includegraphics[width=0.42\textwidth]{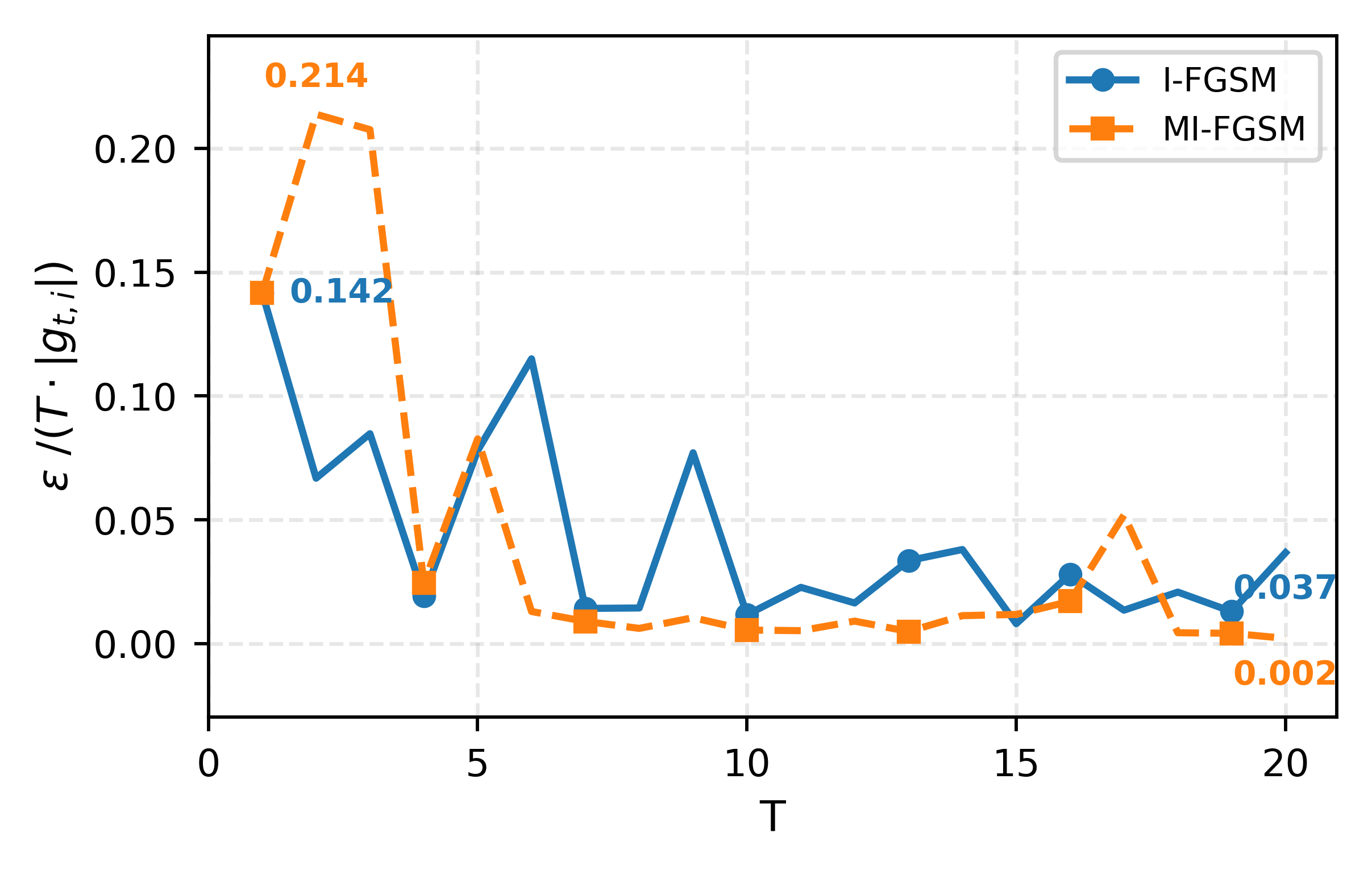}
\caption{Step-size dynamics of a randomly selected coordinate in typical sign-based attacks. $T$ is the maximum iteration. We reformulate $\text{sign}(\bm{g}_t)$ as $\bm{g_{t,i}} / |\bm{g_{t,i}}|$ where $\bm{g}$ denotes the gradient or momentum and $i$ is the number of the randomly selected coordinate. So, the step-size is $\epsilon/(T |\bm{g_{t,i}}|)$ for coordinate-wise gradient.}
\label{fig:instable stepsize}
\end{figure}

To discuss the convergence failures of sign-based gradient, we are naturally led to investigate their step-size scheduling. With regard to the sign-based gradient direction, the step-size is usually set to constant such as $\alpha=\epsilon/T$, where $T$ is the total number of iterations. However, when reformulating I-FGSM as a standard coordinate-wise gradient descent, we reveal that the step-size for each coordinate becomes heavily depends on its gradient magnitude, thereby violating the classical conditions for convergence that require step-size decaying. As visualized in Fig. \ref{fig:instable stepsize}, we plot the progression of a randomly chosen coordinate, the step-sizes in I-FGSM fluctuate without any discernible decaying structure. While the momentum can avoid the fluctuation of an individual gradient, this non-decreasing behavior is similarly pronounced in MI-FGSM, confirming that non-decaying step-size scheduling is an inherent characteristic of sign-based optimization paradigms.

In iterative optimization, a large initial step-size often facilitates rapid progress toward a region near the optimum, whereas a subsequent gradual decay enables finer adjustments for stable convergence. This behavior is particularly important in adversarial attacks, where only a limited number of iterations is performed for the constrained local optimization. It should be indicated that the strategy of using a monotonically decreasing step-size occupies a central position in optimization theory, serving as a foundational component for algorithm design with both theoretical and practical importance. It provides essential convergence guarantees for a broad family of iterative methods. For example, in the convex setting, classical schedules such as $\alpha_t = \alpha / \sqrt{t}$ ensure that the regular gradient descent method achieves an optimal convergence rate of $O(1/\sqrt{t})$ \cite{bertsekas2003convex}. 

A seminal adaptation of monotonically decreasing step-size to the coordinate-wise setting is the work of the adaptive mechanism known as AdaGrad \cite{Duchi2010AdaptiveSM}. Unlike global decaying schedules, AdaGrad adapts each coordinate-wise step-size inversely to the root sum of squares of its past gradients. This construction naturally enforces a monotonically decreasing coordinate-wise step-size (MDCS), allowing AdaGrad to achieve a tighter optimal convergence rate for general convex objectives in sparse learning.

The critical importance of MDCS was underscored by the failure of convergence analysis about Adam \cite{Kingma2015AdamAM} identified in \cite{reddi2019convergence}. Their analysis shows that enforcing MDCS rectifies this issue, leading to the development of AMSGrad. Although Adam has long been the dominant optimizer for training deep neural networks, AMSGrad, despite being a simple variant by incorporating MDCS, not only easily derives theoretical convergence but also frequently outperforms Adam in empirical evaluations \cite{reddi2019convergence}.

The frequent fluctuation of the sign of perturbation has already been observed in \cite{Zhu2024GRA}. Such fluctuations can cause the optimization to stagnate in local optima due to the constant step-size used in sign-based optimizers. To mitigate this issue, they introduced a decay indicator that dynamically adjusts the step-size in response to sign fluctuations, thereby enhancing the overall optimization process \cite{Zhu2024GRA}. 

Building upon the aforementioned analysis, we will incorporate MDCS into state-of-the-art sign-based optimizers to address their non-convergence and instability issues. Our contributions can be summarized as follows,
\begin{itemize}
\item By reformulating sign-based optimizer as a specific coordinate-wise gradient descent, we present an interesting viewpoint from the perspective of optimization, revealing that one cause for the attack instability is their step-size rule. 
\item Based on our new viewpoint,we propose a series of new attack algorithms to enhance stability, in which MDCS is enforced in sign-based optimizers. Typically, we prove that MI-FGSM with MDCS attains its optimal averaging convergence for general convex problems.
\item  Experimental comparisons with state-of-the-art sign-based optimizers on image classification, VQA and image-text retrieval tasks illustrate that our approach not only significantly improves the attack success rates but also enhances their stability.
\end{itemize}
\section{Related Work}
\label{sec:Related Work}

In this section, we describe the optimization problem for adversarial attacks and several typical optimizers.

Let $\mathcal{S}=\{(\bm{x}_1, y_1), \dots, (\bm{x}_m, y_m)\}$ be a training set, where $y_i$ is the label of $\bm{x}_i \in \mathbb{R}^{d} $. %and the sample $(\bm{x}_{i},y_{i})$ is uniformly random chosen from a distribution $\mathcal{D}$.
Given a classifier $f_{\bm{\theta}}$ with a predefined $\bm{\theta}$, generating a non-targeted AE $\bm{x}^{adv}$ from a real example $\bm{x}$ can be formulated as a constrained optimization problem \cite{Madry2018TowardsDL,Goodfellow2015ExplainingAH},
\begin{equation}\label{adv-optimization}
\max J(f_{\theta}(\bm{x}^{adv}),y), \ s. t. \  \|\bm{x}^{adv}-\bm{x}\|_p \leq \epsilon,
\end{equation}
where $J(f_{\theta}(\bm{x}),y)$ is the loss function. Obviously, optimization problem (\ref{adv-optimization}) coincides with our intuition, i.e., adversarial attack is to find an AE that misleads the model prediction (i.e., $ f_{\bm{\theta}}(\bm{x}^{adv})\neq y$) while the $L_p$-norm of the adversarial perturbation $\|\bm{x}^{adv}-\bm{x}\|_p$ should be restricted to a threshold $\epsilon$. As the constrained domain is typically small under the imperceptibility constraint, problem (\ref{adv-optimization}) becomes a local optimization problem, which is a key distinction from conventional optimization settings. In this paper, we only consider the cross-entropy loss and $p=\infty$. For convenience,  we denote $\mathbf{Q}=\{\bm{z}: \| \bm{z}-\bm{x}\|_\infty \leq \epsilon\}$ and  rewrite $J(f_{\theta}(\bm{x}^{adv}),y)$ as $J(\boldsymbol{x}^{adv})$.

%Alternatively, learning adversarial examples can also be described as a {\it regularized} optimization problem \cite{Carlini2017TowardsET}
%\begin{equation}\label{reg-optimization}
%\min \lambda\|\bm{x}^{adv}-\bm{x}\|_p- J(\bm{x}^{adv},y),
%\end{equation}
%where $\lambda$ is the trade-off parameter. In this paper, we only consider the cross-entropy loss and $p=\infty$. For convenience, $J(f_{\theta}(\bm{x}^{adv}),y)$ will be rewritten as $J(\boldsymbol{x}^{adv}_{t})$.

% \subsection{Gradient-based Attacks}
FGSM \cite{Goodfellow2015ExplainingAH} is a basic gradient-based attack. It has only one-step update, i.e.,
\begin{equation}\label{FGSM}
\bm{x}^{adv}=\bm{x}+\epsilon \ \mathrm{sign}(\nabla_{\bm{x}} J(\bm{x})),
\end{equation}
where $\mathrm{sign}(\cdot)$ is the sign function. From (\ref{FGSM}), it is easy to know $\|\bm{x}^{adv}-\bm{x}\|_\infty \leq \epsilon$.

%\textbf{Iterative Fast Gradient Sign Method (I-FGSM).}
I-FGSM \cite{Kurakin2017AdversarialEI} is a FGSM with multiple iterative steps, 
\begin{equation}\label{I-FGSM}
\bm{x}_{t+1}^{adv}= \bm{x}_{t}^{adv} + \alpha \ \mathrm{sign}(\nabla_{\bm{x}} J(\bm{x}_{t}^{adv})),
\end{equation}
where $\bm{x}^{adv}_0 = \bm{x}$. The step-size $\alpha$ is set to $\epsilon/T$ so that $\|\bm{x}_{t}^{adv}-\bm{x}\|_\infty \leq \epsilon,\  0 \leq \forall t\leq T$, where $T$ is the total number of iterations. It should be noted that $\alpha$ only determines the magnitude of the update along $\mathrm{sign}(\nabla_{\bm{x}} J(\bm{x}_{t}^{adv}))$ rather than along its real gradient $\nabla_{\bm{x}} J(\bm{x}_{t}^{adv})$.

PGD \cite{Madry2018TowardsDL} is also multiple iterative steps FGSM but starting from a random starting point, i.e.,
\begin{equation}\label{PGD}
\bm{x}_{t+1}^{adv} = \mathrm{Clip}_{\bm{x}}^\epsilon \left[\bm{x}_{t}^{adv} + \alpha \ \mathrm{sign} (\nabla_{\bm{x}} J(\bm{x}_{t}^{adv}))\right],
\end{equation}
As the step-size is not restricted, a clipping (\ref{clip}) operation is used to ensure that $\|\bm{x}^{adv}-\bm{x}\|_\infty \leq \epsilon$. For an image $\bm{x}=(x_1,x_2,x_3)$ which is 3-D tensor, its clipping is \cite{Kurakin2017AdversarialEI}
\begin{equation}\label{clip}
\begin{aligned}
& \mathrm{Clip}_{\bm{x}}^\epsilon \left[\bm{x}(x_1,x_2,x_3)\right] =\min \{ 255, \bm{x}(x_1,x_2,x_3)+\epsilon,\\
& \max\{0, \bm{x}(x_1,x_2,x_3)-\epsilon, \bm{x}(x_1,x_2,x_3)\}\}.
\end{aligned}
\end{equation}

MI-FGSM \cite{Dong2018BoostingAA} integrates HB momentum \cite{Polyak1964SomeMO} into the iteration of I-FGSM. Its update is
\begin{equation}\label{MI}
\begin{array}{l}
\bm{m}_{t+1}=\mu \ \bm{m}_t  +  \displaystyle\frac{\nabla_{\bm{x}} J(\bm{x}_{t}^{adv})}{\|\nabla_{\bm{x}} J(\bm{x}_{t}^{adv})\|_1}\\
\bm{x}_{t+1}^{adv}=\mathrm{Clip}_{ \bm{x}}^\epsilon \left(\bm{x}_{t}^{adv} + \alpha \ \mathrm{sign}(\bm{m}_{t+1})\right)
\end{array},
\end{equation}
where $\alpha=\epsilon/T$. Note that $\bm{m}_{t+1} $ is the sum of weighted past gradients and then this iterative direction is more stable than the individual $\mathrm{sign}(\nabla_{\bm{x}} J(\bm{x}_{t}^{adv}))$ in I-FGSM and PGD. 

Based upon MI-FGSM, a number of momentum-based attack methods have been developed by employing more techniques from optimization theory. For instance, NI-FGSM \cite{Lin2020NesterovAG} integrates Nesterov's momentum, VMI-FGSM \cite{Wang2021BoostingTT} introduces variance tuning, and GRA \cite{Zhu2024GRA} incorporates variance tuning and decay indicator techniques. Recent studies suggest that AEs located at flat maxima of the loss landscape tend to exhibit stronger transferability. For example, 
% RAP \cite{Zhu2024RAP} explicitly searches for AEs in flat loss regions, but at the cost of increased computation. Subsequent methods such as 
PGN \cite{Ge2023Boosting} and MEF \cite{Qiu2024Enhancing} mitigate this issue by incorporating a gradient norm penalty into the loss.

Another line of work focuses on improving transferability through input transformation. DI \cite{Xie2019ImprovingTO} pioneered this direction by applying random resizing and padding to inputs. This was followed by techniques such as TI \cite{Dong2019EvadingDT}, SI \cite{Lin2020NesterovAG}, Admix \cite{Wang2021admix}, and BSR \cite{wang2024bsr}. Such transformation strategies are commonly integrated with optimization-based attacks. Recently, OPS \cite{guo2025boosting} improves optimization efficiency by randomly selecting transformations at each iteration, thereby promoting gradient diversity during the attack process.

%For attacking VLPs, the work in \cite{zhang2022towards} has demonstrated that a naive strategy of independently perturbing image and text modalities is fundamentally flawed. Without coordinating the adversarial perturbations to account for cross-modal interactions, the unimodal attacks may conflict with each other.  By considering the consistency between attacks of different modalities, Co-Attack collaboratively combines multimodal perturbations towards a stronger adversarial attack.T

Recent efforts to generate transferable AEs for VLMs have progressively focused on enhancing sample diversity. Initial works like Co-Attack \cite{zhang2022towards} established a baseline with collaborative multimodal perturbations. Subsequently, SGA \cite{lu2023sga} employed online data augmentation, which was refined by DRA \cite{gao2024dra} through sampling within the adversarial trajectory's intersection to mitigate overfitting. SA-AET \cite{jia2025semantic} further leverage an adversarial evolution triangle for more principled diversification while simultaneously crafting attacks in a semantic-aligned subspace to reduce model dependency.  Despite the increasing sophistication in handling multimodal interactions, all these methods retain PGD as their underlying optimizer to generate AEs.

\section{The Proposed MDCS}

In this section, we first reformulate I-FGSM as a coordinate-wise gradient descent and explicitly describe its step-size scheme, then we propose our MDCS strategy.

Let $x_{t+1,i}^{adv}$ be the $i$-th component of $\bm{x}_{t+1}^{adv}$ generated by I-FGSM, i.e.,
$$
x_{t+1,i}^{adv}= x_{t,i}^{adv} + \alpha \ \mathrm{sign} (\frac{\partial J(\bm{x}_{t}^{adv})}{\partial x_i}).
$$
Then, we have
\begin{equation}\label{I-FGSM-i}
x_{t+1,i}^{adv}= x_{t,i}^{adv} +\frac{\alpha}{|\frac{\partial J(\bm{x}_{t}^{adv})}{\partial x_i}| }
\frac{\partial J(\bm{x}_{t}^{adv})}{\partial x_i}.
\end{equation}
According to optimization theory, the coordinate-wise step-size in I-FGSM is formally defined as $\frac{\alpha}{|\partial J(\bm{x}_{t}^{adv})/\partial x_i|}$. This formulation implies that the step-size is governed by the gradient magnitude. As the attack converges near a local optimum, the gradient approaches zero, causing the step-size to diverge and exhibit unexpected fluctuations, a behavior we depict in Fig. \ref{fig:instable stepsize}. 

The step-size analysis for I-FGSM is fully applicable to all the sign-based optimizers. Specifically, the coordinate-wise step-size of MI-FGSM is $\frac{\alpha}{|{m}_{t+1, i}|}$. The key distinction lies in the fact that MI-FGSM enjoys better stability owing to momentum being a weighted sum of past gradients.

Note that the regular Adam \cite{Kingma2015AdamAM} for solving problem (\ref{adv-optimization}) is
\begin{equation}\label{adam}
\begin{array}{l}
\bm{m}_{t}=\beta_{1t} \ \bm{m}_{t-1}  + (1-\beta_{1t}) \nabla_{\bm{x}}{J}(\bm{x}_{t}),\\
\bm{V}_{t}=\beta_{2t}\bm{V}_{t-1}+(1-\beta_{2t})\text{diag}(\nabla_{\bm{x}}{J}(\bm{x}_{t}) {\nabla_{\bm{x}}{J}(\bm{x}_{t})}^{\top}),\\
\bm{x}_{t+1} = P_{\mathbf{Q},\bm{V}^{-1}_{t}}(\bm{x}_{t} - \alpha_t  \bm{V}_{t}^{-\frac{1}{2}}\bm{m}_{t}),
\end{array}
\end{equation}
where the hyper-parameters $\beta_{1t},\beta_{2t} \in [0, 1)$ control the exponential decay rates of the moving averages (EMA), and $\bm{V}_{t}$ is a $d\times d$ diagonal matrix. For a positive definite matrix $\bm{V}_{t}^{-1}$, the weighted $L_2$-norm (Mahalanobis norm) is defined by $\|\bm{x}\|_{\bm{V}_{t}^{-1}}^{2} = \bm{x}^{\top}\bm{V}_{t}^{-1}\bm{x}$. The projection $P_{\mathbf{Q},\bm{V}^{-1}_{t}}(\bm{x})$  of $\bm{x}$ onto $\mathbf{Q}$ is defined by 
$$P_{\mathbf{Q},\bm{V}^{-1}_{t}}(\bm{x})=\min_{\bm{z} \in \mathbf{Q}} \|\bm{x} - \bm{z} \|_{\bm{V}^{-1}_{t}}=\min_{\bm{z} \in \mathbf{Q}} \|\bm{V}^{-\frac{1}{2}}(\bm{x} - \bm{z}) \|.$$

The primary challenge in establishing the convergence guarantee for Adam-type optimizers \cite{Kingma2015AdamAM} stems from their use of EMA, which can violate the standard step-size decay condition required for conventional theoretical analysis \cite{reddi2019convergence}. To resolve this theoretical issue, Reddi at al. \cite{reddi2019convergence} demonstrate that provable convergence can be readily recovered by enforcing MDCS, i.e.,
\begin{equation}\label{AMSgrad}
\begin{array}{l}
\bm{m}_{t}=\beta_{1t} \ \bm{m}_{t-1}  + (1-\beta_{1t}) \nabla_{\bm{x}}{J}(\bm{x}_{t}), \\
V_{t}=\beta_{2t}V_{t-1}+(1-\beta_{2t})\text{diag}(\nabla_{\bm{x}}{J}(\bm{x}_{t}) {\nabla_{\bm{x}}{J}(\bm{x}_{t})}^{\top}),\\
\hat{V_t} = \max\{\hat{V}_{t-1}, V_{t}\}\\
\bm{x}_{t+1} = P_{\mathbf{Q},\bm{\hat{V}}^{-1}_{t}}(\bm{x}_{t} - \alpha_t  \bm{\hat{V}}_{t}^{-\frac{1}{2}}\bm{m}_{t}).
\end{array}
\end{equation}
Such an algorithm is called AMSGrad \cite{reddi2019convergence}. It has been proved to attain the optimal convergence while deriving better empirical performance in image classification \cite{reddi2019convergence}.

Motivated by AMSGrad, we propose a broad class of optimizers based on available sign-based attacks, which ensures that the coordinate-wise step-size is monotonically decreasing. Specifically, the detailed steps of MI-FGSM with MDCS are shown in Algorithm \ref{alg:MDCS}.
\begin{algorithm}
\caption{MDCS-MI}
\label{alg:MDCS}
    \begin{algorithmic}[1]
    \Require
    Loss function $J$, a raw example $\bm{x}$ with ground-truth label $y$, the perturbation size $\epsilon$, momentum parameter $0 < \beta_t \leq 1$, step-size $\alpha_t > 0$, maximum iteration $T$.
    \State Initialize $\bm{x}^{adv}_0=\bm{x}$, $\bm{m}_1=\bm{0}$,  $d_{0,i} = 1$.
    \Repeat
    \State $\bm{m}_{t+1}=\beta_{t} \ \bm{m}_{t}  + \displaystyle\frac{\nabla_{\bm{x}} J(\bm{x}_{t}^{adv})}{\|\nabla_{\bm{x}} J(\bm{x}_{t}^{adv})\|_1}$, 
    \State $d_{t,i} = \mathrm{min} (\frac{1}{|\bm{m}_{t+1,i}| }, d_{t-1,i} )$, 
    \State $\bm{D}_t = \mathrm{diag}(\bm{d}_{t})$,
    \State $\bm{x}_{t+1}^{adv}= P_{\mathbf{Q}, \bm{D}^{-1}_t} (\bm{x}_{t}^{adv} + \alpha_t  \bm{D}_t \bm{m}_{t+1}).$
    \Until {$t = T$}
    \Ensure  $\bm{x}_{T}^{adv}$.
    \end{algorithmic}
\end{algorithm}

For an image $\bm{x}=(x_1,x_2,x_3)$ which is 3-D tensor,
\begin{equation}\label{proj}
\begin{aligned}
\bm{x}_{t+1}^{adv}
&= P_{\mathbf{Q}, \bm{D}^{-1}_t} (\bm{x}_{t}^{adv} + \alpha_t  \bm{D}_t \bm{m}_{t+1}) \\
&=\mathrm{Clip}_{\bm{x}}^\epsilon \left[\bm{D}_{t}^{-\frac{1}{2}}(\bm{x}_{t}^{adv} + \alpha_t  \bm{D}_t \bm{m}_{t+1})\right].
\end{aligned}
\end{equation}
It is easy to find that $d_{t,i} \leq d_{t-1,i} \leq 1, \ \forall t>0$. Obviously, the main difference from MI-FGSM is that we employ MDCS instead of the sign operation. To conduct convergence analysis, some assumptions are required. The detailed proof is given in Supplementary Material.

\begin{assumption}\label{ass:gfinite}
Assume that there exists a constant $M > 0$ such that
\begin{equation*}
\|\nabla_{\bm{x}}{J}(\bm{x}) \|_1 \leq M,  \ \forall \bm{x} \in \mathbf{Q}.
\end{equation*}
\end{assumption}

\begin{assumption}\label{ass:wfinite}
Assume that there exists a constant $G > 0$ such that
\begin{equation*}
  \| \boldsymbol{x}_1 - \boldsymbol{x}_2 \| \leq G,  \forall \bm{x}_1, \bm{x}_2 \in \mathbf{Q}.
\end{equation*}
\end{assumption}

\begin{theorem}
\label{thm:bigtheorem}
Suppose the objective function $J(\boldsymbol{x})$ is concave on $\mathbf{Q}$ and $\boldsymbol{x}^{\ast} $ is a solution of problem (\ref{adv-optimization}). Let Assumption \ref{ass:gfinite} and \ref{ass:wfinite} hold and let $\{ \boldsymbol{x}^{adv}_{t}\}_{t=1}^{T}$ be generated by \cref{alg:MDCS}. Assume $0<\beta<1$, $0<\lambda<1$ and $\gamma > 0$. Let $\beta_t = \beta \lambda^{t-1}$ and $\alpha_t = \frac{\gamma}{\sqrt{t}}$. Then we have
\begin{align*}
J(\boldsymbol{x}^\ast)-J(\boldsymbol{\bar{x}}^{adv}_T)
& \leq O(\frac{1}{\sqrt{T}})
\end{align*}
where $\boldsymbol{\bar{x}}^{adv}_T = \frac{1}{T}\sum_{t=1}^{T}\boldsymbol{x}^{adv}_{t}$.
\end{theorem}

It is necessary to give some remarks.

\begin{itemize}
\item \cref{thm:bigtheorem} indicates that MDCS-MI achieves optimal convergence rate $O(1/\sqrt{T})$, which avoids the non-convergence issue inherent in MI-FGSM caused by the sign-based operation.

\item Our theoretical analysis can be readily extended to other sign-based optimizers such as I-FGSM and PGD. Due to using MDCS, the proof here is similar to that of AdaGrad and AMSGrad in spirit.

\item If an AE $\boldsymbol{x}^{adv}$ is successfully generated from a correctly classified sample $\boldsymbol{x}$, it holds that $J(\boldsymbol{x}^{adv}) > J(\boldsymbol{x})$. The human-imperceptibility constraint restricts $\boldsymbol{x}^{adv}$ to a confined neighborhood of $\boldsymbol{x}$, thus justifying the mild assumption that the objective function is locally concave.

\end{itemize}

\section{Experiments}

In this section, we show the efficacy of our MDCS strategy through extensive experiments on image classification models and VLMs. Typically, we only focus on comparing with some recent optimization-based methods under untargeted attack setting, using their standard hyperparameters. Transferability is quantified by the attack success rate, with a higher rate denoting more effective and transferable attacks. All experiments were performed on Hygon K100-AI 64GB DCUs. Due to the space limitation, more details and experimental results are given in \textbf{Supplementary Material}.

% The Supplementary Material is organized as follows: \cref{app:appendix} presents the convergence analysis of MDCS-MI; \cref{app:8.1.1} provides the pseudocode for MDCS-MEF and MDCS-OPS; \cref{app:8.1.2} details the ablation studies; and \cref{app:8.1.3} offers visualizations and imperceptibility.

\subsection{Evaluation on Image Classification Tasks}

Since our MDCS-MI is directly established upon MI-FGSM, we will conduct the same experiments as that in \citep{Dong2018BoostingAA}. Following \cite{Dong2018BoostingAA} and \cite{Lin2020NesterovAG}, we randomly sample 1000 images from NIPS2017 dataset.
% \footnote{{\url{https: //www.kaggle.com/competitions/nips-2017-non-targeted-adversarial-attack/data}}}
The normally trained models are chosen from both branches of CNNs and ViTs for black-box attacks, including ResNet-50 (Res50), VGG-16, MobileNet-v2 (Mob-v2), Inception-v3 (Inc-v3) in CNN branch; ViT-Base-patch16 (ViT-B/16) \cite{Alexey2021Vit}, PiT-Base (PiT-B) \cite{heo2021pit}, Visformer-Small (Vis-S) \cite{chen2021visformer} in ViT branch. Following \cite{Li2023Improving}, we also collect robust Inc-v3$_{ens3}$, Inc-v3$_{ens4}$, IncRes-v2$_{ens}$) and consider AT, HGD, RS, Bit-Red as defensed models.

The hyperparameters in problem (\ref{adv-optimization}) are fixed. We set $\epsilon={16}/{255}$ and $T = 10$. As usual, $\lambda = 0.999$ and $\beta = 0.999$ \cite{wang2019sadam, Tao2021TheRO, reddi2019convergence}. For convenience, we select $\alpha_t={\epsilon}\gamma/{T}$ in our MDCS-type algorithms. $\gamma$ is determined by simple grid search over the range [2, 4] to balance transferability and imperceptibility.

\subsubsection{Evaluation of Attack Stability}

To illustrate the benefits brought by our MDCS strategy, we examine the correlation between success rates and the maximum iteration $T$. As depicted in Fig. \ref{fig:stabilty}, integrating HB momentum into the iterative direction in MI-FGSM does contribute to stability. %Notably, when $T$ is small, MI-FGSM might exhibit superior performance to MDCS-MI due to its larger step-size.
Nonetheless, MDCS can improve transferability of MI-FGSM across different neural network architectures while consistently maintaining better stability.

\begin{figure*}[htbp]
\centering
\includegraphics[width=0.48\textwidth]{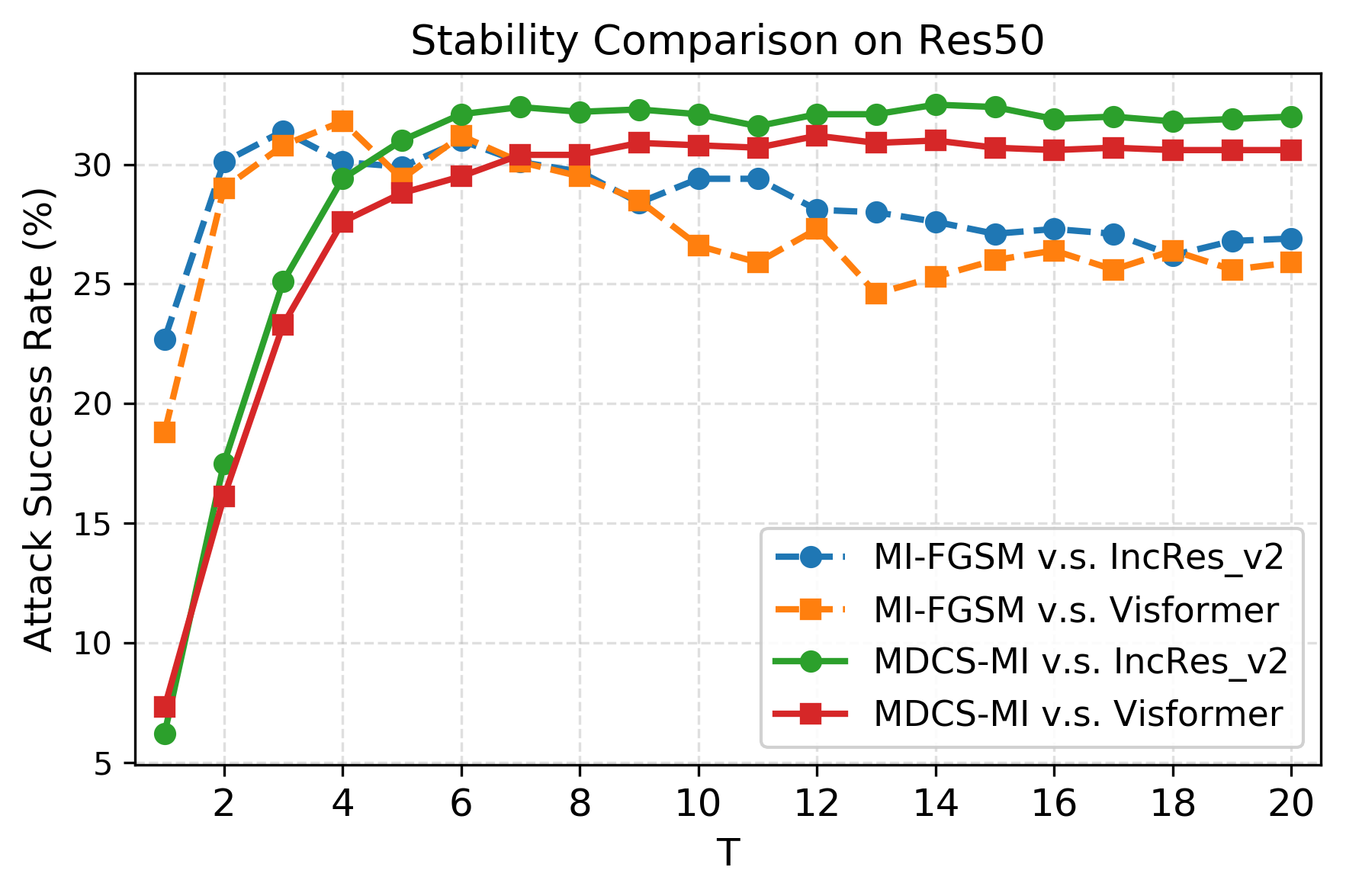}
\includegraphics[width=0.48\textwidth]{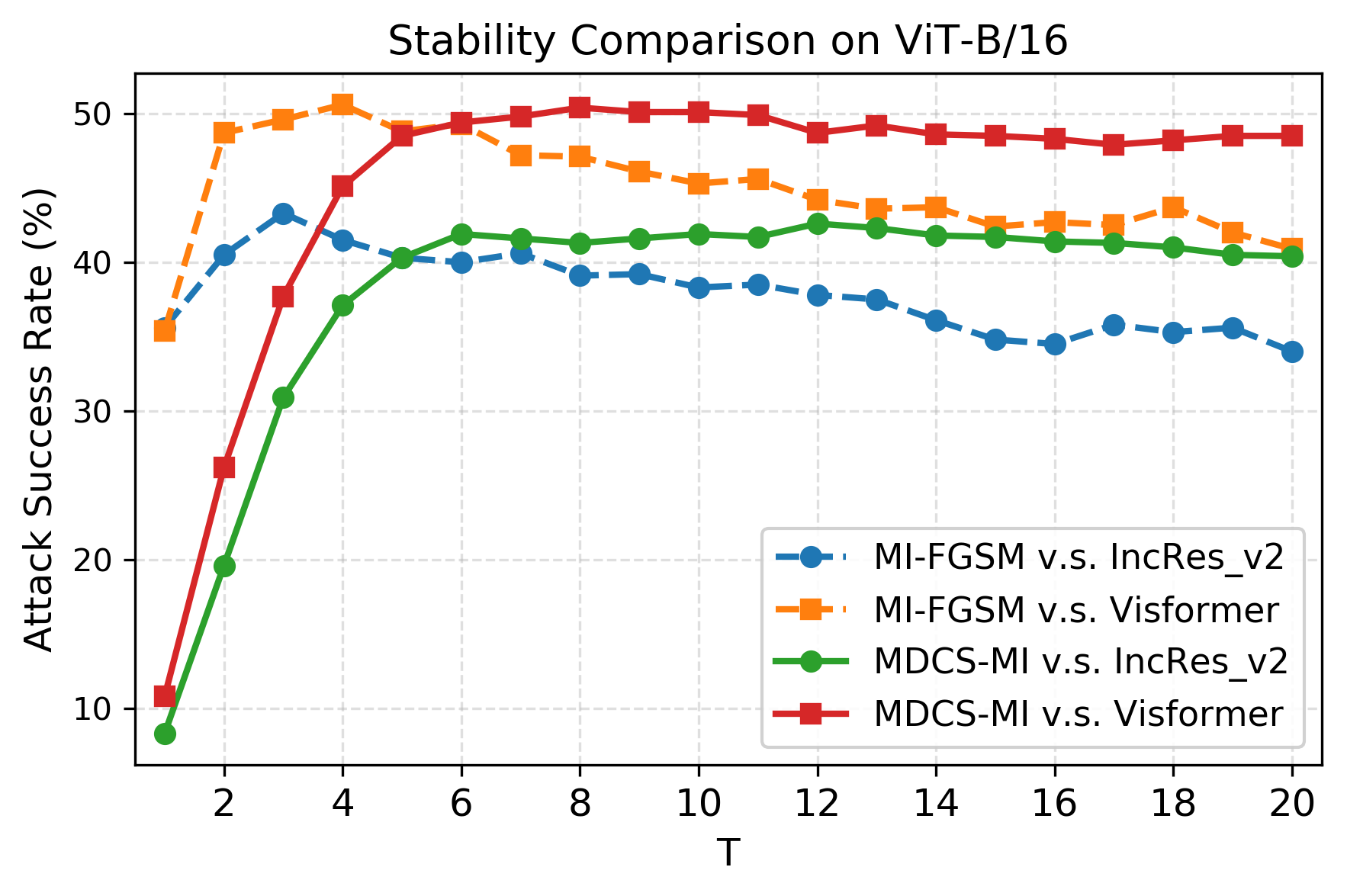}
\caption{Stability comparison of transfer attacks between MDCS-MI and MI-FGSM. AEs are crafted for Res50 and ViT-B/16, respectively.}
\label{fig:stabilty}
\end{figure*}

\subsubsection{Attacks on Normally Trained Models}

In this subsection, we evaluate the adversarial transferability of the proposed attack on normally trained models including both single and ensemble architectures. As far as we know, OPS \cite{guo2025boosting} currently achieves highly competitive performance compared to typical baseline attacks such as MI-FGSM \cite{Dong2018BoostingAA}, VMI-FGSM \cite{Wang2021BoostingTT}, PGN \cite{Ge2023Boosting} and MEF \cite{Qiu2024Enhancing}. Based on these findings and considering the central role of OPS as a state-of-the-art approach, we focused our subsequent comparative experiments primarily on MI-FGSM, MEF and OPS. In addition to single-model attacks, we evaluate our MDCS strategy in ensemble-model settings. Specifically, we focus on attacking an ensemble of equally-weighted models using logit-based integration. For comparative baselines, we include SVRE \cite{Xiong2022svre} and AdaEA \cite{Chen2021AdaEA}. 

As shown in Tab. \ref{tab:normal_attack} and \ref{tab:ensemble}, MDCS-MI consistently surpasses MI-FGSM in both single-model and ensemble attacks.  Moreover, integrating MDCS into other attacks also improves their transferability, with MDCS-OPS achieving the best performance. Visualizations in Fig. \ref{fig:vis_cv} confirm that the adversarial perturbations generated by our methods (MDCS-MI, MDCS-MEF, MDCS-OPS) remain imperceptible, thereby satisfying the requirement for stealth in practical scenarios.

\begin{table*}[htbp]\small
\captionsetup{font={normalsize}}
\caption{\normalsize{Transferability comparisons across different network architectures. AEs are crafted for Res50. The baseline attacks are divided into 3 categories: \ding{172} gradient-based attacks; \ding{173} momentum-based attacks; \ding{174} input-transformation attacks. The best results are marked in bold. The gray area represents adversarial attacks under the white-box setting.}}
\label{tab:normal_attack}
\centering
\begin{tabular}{c|c|ccccccc}%
\toprule %[2pt]
Type&Attack&Res50& VGG16 & Mob-v2 & Inc-v3 & ViT-B & PiT-B& Vis-S \\
\midrule %[2pt]
\multirow{2}{*}{\ding{172}}&AutoAttack \cite{Croce2020ReliableEO}& \cellcolor[rgb]{.827,.827,.827}99.5& 34.0& 30.6&13.9& 3.0& 8.2& 11.3\\ 
&I-FGSM \cite{Goodfellow2015ExplainingAH}& \cellcolor[rgb]{.827,.827,.827}99.7&36.3&34.1&18.5&6.8& 11.8& 14.6\\
\hline
\multirow{8}{*}{\ding{173}}&MI \cite{Dong2018BoostingAA}& \cellcolor[rgb]{.827,.827,.827}100.0& 59.4& 53.1&36.3& 12.5& 23.1& 26.6\\
&\bf{MDCS-MI} & \cellcolor[rgb]{.827,.827,.827}100.0&  67.2& 60.3&41.3& 13.4& 23.3&30.8\\
&VMI \cite{Wang2021BoostingTT}& \cellcolor[rgb]{.827,.827,.827}99.8& 69.2& 66.3&56.8& 31.0& 46.9&54.5\\
&GRA \cite{Zhu2024GRA}&\cellcolor[rgb]{.827,.827,.827}97.3& 85.2&  84.3&81.3& 44.6& 62.2& 72.9\\
&MUMODIG \cite{ren2025mumodig}&\cellcolor[rgb]{.827,.827,.827}98.9& 86.4& 84.5&80.2& 46.3& 65.5& 75.6\\
&PGN \cite{Ge2023Boosting}& \cellcolor[rgb]{.827,.827,.827}98.9& 88.4& 86.9&85.4& 49.1& 68.2& 76.6\\
&MEF \cite{Qiu2024Enhancing}& \cellcolor[rgb]{.827,.827,.827}99.3& 94.9& 94.4&91.2&65.3&81.1&88.2\\
&\bf{MDCS-MEF}& \cellcolor[rgb]{.827,.827,.827}100.0&96.4 & 95.5 & 93.4 & 58.7 & 78.8 & 91.0 \\ \hline
% \multirow{4}{*}{\ding{173}}&TRAP& 99.2$^*$& 95.9& 94.3&84.0& 27.3& 46.6& 63.5\\
% & BFA& 98.2$^*$& 93.6& 92.4& 85.0& 50.0& 72.8& 85.8\\
% &P2FA & 100.0$^*$&  98.3& 97.9&84.1& 37.0& 66.6&85.7\\
% &\bf{P2FA+MDCS}& 100.0$^*$&  & && & &\\ \hline
% \multirow{3}{*}{\ding{174}}&BPA& 96.3$^*$&  96.9& 94.9&85.4& 32.7& 49.7&74.3\\
% &DRA& 94.0$^*$&  96.6& 97.8&95.1& 73.0& 77.8&89.3\\
% &\bf{MDCS-DRA}& 95.3$^*$&  97.1& 98.2&95.5& 73.5& 78.0&89.9\\
\multirow{5}{*}{\ding{174}}&TI \cite{Dong2019EvadingDT}&\cellcolor[rgb]{ .827,  .827,  .827}97.8&57.9&	46.9&	38.9&	15.3&	16.5&	23.2\\
&DI \cite{Xie2019ImprovingTO}& \cellcolor[rgb]{.827,.827,.827}98.7&	71.0&	66.2&	57.1&	27.5&	39.7&	49.5\\
&SIA \cite{wang2023structure}& \cellcolor[rgb]{.827,.827,  .827}99.5& 94.5&  92.6 &  84.3&  53.6 &  77.2 &  86.5\\
 & BSR \cite{wang2024bsr}& \cellcolor[rgb]{.827,.827,.827}99.1& 97.2& 96.2& 89.7& 56.0& 80.8&89.2\\
&OPS \cite{guo2025boosting} & \cellcolor[rgb]{.827,.827,.827}99.5& 98.0& 	97.8& 98.2& 88.8& 93.8& 96.7\\
&\bf{MDCS-OPS}& \cellcolor[rgb]{.827,.827,.827}99.9&\bf{98.9} &\bf{99.0} &\bf{99.1} & \bf{89.3} & \bf{94.7}& \bf{97.9}\\
\hline
\end{tabular}
\end{table*}

\begin{table}\small
\captionsetup{font={normalsize}}
\caption{\normalsize{The success rates (\%) of adversarial attacks against ensemble models.}}
\label{tab:ensemble}
\centering
\begin{tabular}{c|cccc}%
\toprule %[2pt]
Attack&Inc-v3& Res34& Inc-v4& IncRes-v2\\	
\midrule %[2pt]
SVRE \cite{Xiong2022svre} & \cellcolor[rgb]{ .827,.827,.827}99.5& \cellcolor[rgb]{ .827,  .827,  .827}100.0& 74.6 &65.9\\
AdaEA \cite{Chen2021AdaEA} & \cellcolor[rgb]{ .827,  .827,  .827}99.6& \cellcolor[rgb]{ .827,  .827,  .827}100.0& 72.4&64.7\\
MI & \cellcolor[rgb]{ .827,  .827,  .827}99.6& \cellcolor[rgb]{ .827,  .827,  .827}100.0&75.6&  67.0\\
\bf{MDCS-MI}&  \cellcolor[rgb]{ .827,  .827,  .827}99.9& \cellcolor[rgb]{ .827,  .827,  .827}100.0&79.3& 73.5\\
MEF& \cellcolor[rgb]{ .827,  .827,  .827}99.9& \cellcolor[rgb]{ .827,  .827,  .827}100.0& 97.1& 95.2\\
\bf{MDCS-MEF}&  \cellcolor[rgb]{ .827,  .827,  .827}100.0& \cellcolor[rgb]{ .827,  .827,  .827}100.0& 97.9&96.6\\
OPS& \cellcolor[rgb]{ .827,  .827,  .827}99.9& \cellcolor[rgb]{ .827,  .827,  .827}100.0&99.7 &99.6 \\
\bf{MDCS-OPS}&  \cellcolor[rgb]{ .827,  .827,  .827}100.0& \cellcolor[rgb]{ .827,  .827,  .827}100.0& \bf{99.9}& \bf{99.8}\\
\hline
\end{tabular}
\end{table}
% \vspace{-0.5\baselineskip}

\subsubsection{Attacks on Robust Models}

In this subsection, we evaluate the performance of the crafted AEs in defense and adversarial training scenarios. 
% Following \cite{Yuan2022Natural} and \cite{Long2022Frequency}, we consider both defenses and adversarially trained models. 
For simplicity, AEs are crafted for Res50. The success rates are reported in Tab. \ref{tab:defense}. It can be observed that the success rates of typical adversarial attacks exhibit varying degrees of improvement after incorporating our MDCS strategy. Specifically, MDCS-OPS achieves the best results, which validate the effectiveness of the proposed MDCS strategy in defense scenarios.

\begin{table*}\small
\captionsetup{font={normalsize}}
\caption{\normalsize{The success rates (\%) of adversarial attacks against defensed models and adversarially trained models.}}
\label{tab:defense}
\centering
\begin{tabular}{c|cccccccc}%
\toprule %[2pt]$
Attack  &  Inc-v3$_{ens3}$& Inc-v3$_{ens4}$& IncRes-v2$_{ens}$&  AT& JPEG& HGD&RS &Bit-Red\\	
\midrule %[2pt]
% I-FGSM & 16.7 & 8.7 & 10.3 & 8.9&& &&&\\
MI& 22.0&22.5&14.9& 40.7& 50.2& 17.4&27.5 &38.9\\
\bf{MDCS-MI}&23.7&24.2&16.3& 41.1& 52.2& 16.2&27.8 &42.6\\
MEF&73.6& 71.6& 66.3& 45.4& 86.2& 74.5&37.4 &81.2\\
\bf{MDCS-MEF}&74.5& 71.5&66.3& 45.5& 88.5& 74.6&37.2&84.7\\
OPS&96.8& 96.7& 95.3& 58.3& 98.5& 97.2&65.3&98.1\\
\bf{MDCS-OPS}&\bf{97.5}& \bf{97.5}& \bf{96.0}& \bf{58.5}& \bf{98.9}& \bf{98.0}& \bf{66.6}&\bf{99.0}\\
\hline
\end{tabular}
\end{table*}

\subsubsection{Ablation Study}

While both $\epsilon$ and $\gamma$ in MDCS-MI critically affect the performance, we focus here on ablating the perturbation budget $\epsilon$. A detailed analysis of $\gamma$ is deferred to \cref{app:8.1.3}.

Fig. \ref{fig:eps} presents a comparative analysis of two baseline attacks, MI-FGSM and OPS, against their MDCS-enhanced counterparts. As expected, the performance for all methods demonstrates a strong positive correlation with the magnitude of $\epsilon$. More critically, the integration of our MDCS module yields a consistent and significant performance enhancement across the entire range of $\epsilon$ values for both backbones. MDCS-MI consistently outperforms MI-FGSM, and similarly, MDCS-OPS maintains a clear advantage over OPS. 
% This performance gap is particularly highlighted in the inset, which provides a magnified view at $\epsilon = 16/255$, confirming the superiority of MDCS-OPS even when the attack success rate is near saturation. 
The inset at $\epsilon = 16/255$ highlights this superiority even near saturation. These results affirm that the effectiveness of our MDCS module is not confined to a specific perturbation budget but offers robust improvements, underscoring its general applicability and effectiveness in boosting adversarial attacks.

\begin{figure}[htbp]
\centering
\includegraphics[width=0.48\textwidth]{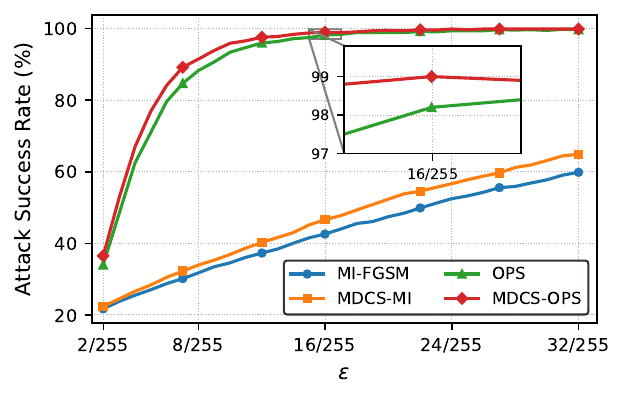}
\caption{Ablation study on the value of $\epsilon$. We employ Res50 as surrogate model and Inc-v3 as target model.}
\label{fig:eps}
\end{figure}

\begin{figure}[htbp]
\setlength{\belowcaptionskip}{-0.2cm}
\centering
\includegraphics[width=0.48\textwidth]{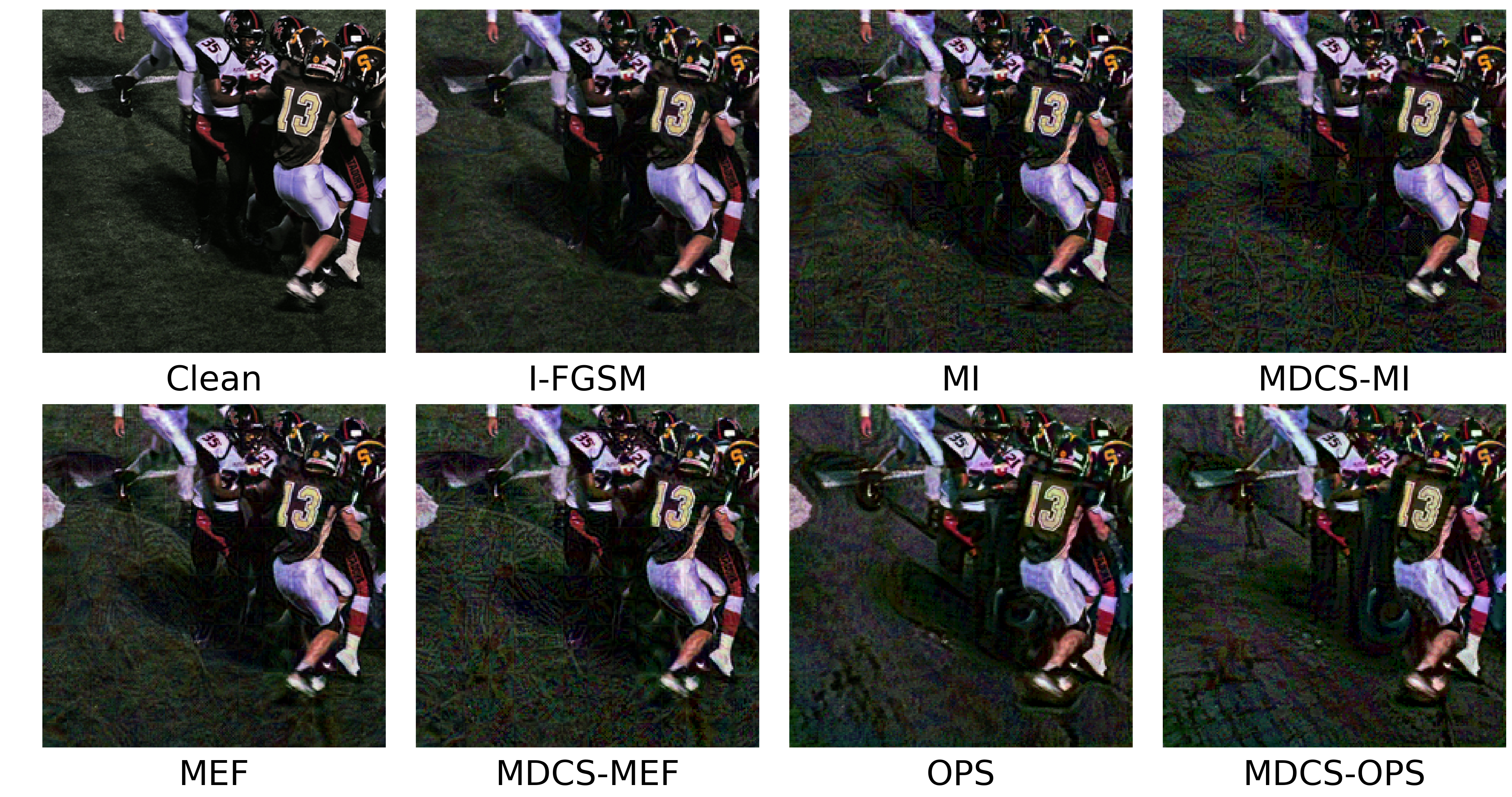}
\caption{Visualization of AEs on image classification tasks. AEs are crafted on ViT-B/16 with $\epsilon = 16/255$.}
\label{fig:vis_cv}
\end{figure}

\subsection{Evaluation on Cross-Modal Retrieval Tasks}

Image-Text cross-modal retrieval addresses the task of retrieving the most relevant top-ranked results from a gallery of one modality, given a query from a different modality \cite{wang2019camp,chen2020imram}. To demonstrate the broader applicability of our MDCS strategy, we further investigate its potential to enhance state-of-the-art multimodal attacks on retrieval tasks. 

Following \cite{gao2024dra,lu2023sga}, we focus on assessing the transferability of AEs across widely adopted VLM architectures: fused and aligned, including ALBEF, TCL, $\text{CLIP}_\text{CNN}$, $\text{CLIP}_\text{ViT}$. We leverage Flickr30K
% \footnote{{\url{https://huggingface.co/datasets/nlphuji/flickr30k}}} 
and MSCOCO
% \footnote{{\url{https://cocodataset.org/}}}
datasets \cite{young2014image, lin2014microsoft} for evaluating the cross-modal retrieval task. We adopt PGD, Sep-Attack (separate unimodal
attack), Co-Attack \cite{zhang2022towards}, SGA \cite{lu2023sga}, DRA \cite{gao2024dra}, SA-AET \cite{jia2025semantic} as our baselines. It can be noticed that these baseline frameworks still rely on conventional sign-based optimizers (e.g., PGD) for the underlying perturbation generation. We hypothesize that replacing the standard optimizer with our theoretically grounded MDCS can improve the stability of the optimization process within these frameworks, thereby unlocking greater black-box transferability. To validate this, we integrate our MDCS strategy into the SGA, DRA, and SA-AET frameworks separately, creating enhanced variants: MDCS-SGA, MDCS-DRA, and MDCS-SAAET.

\begin{table*}[htbp]
\captionsetup{font={normalsize}}
\caption{\normalsize{The adversarial attacks against multimodal models on Flickr30K dataset. The source model column shows VLMs we use to generate multimodal AEs. The gray area represents adversarial attacks under a white-box setting; the rest are black-box attacks. R@1 denotes the attack success rate at the top-1 rank. The adversarial perturbation is set to 8/255, with a per-step size of 2/255, and a total of 10 iterations. For both Image-Text Retrieval (TR) and Text-Image Retrieval (IR), we provide R@1 scores here. More results are listed in the Supplementary Material (\cref{app: retrieval}).}}
\label{tab:blackbox-vqa}
    \centering
    \small
    \begin{tabular}{c|c|cc|cc|cc|cc}
    \toprule
    \multirow{2}[4]{*}{Source} & \multirow{2}[4]{*}{Attack} & \multicolumn{2}{c}{ALBEF} & \multicolumn{2}{c}{TCL} & \multicolumn{2}{c}{CLIP$_{\text{ViT}}$} & \multicolumn{2}{c}{CLIP$_{\text{CNN}}$} \\
\cmidrule{3-10}      &   & TR R@1 & IR R@1 & TR R@1 & IR R@1 & TR R@1 & IR R@1 & TR R@1 & \multicolumn{1}{c}{IR R@1} \\
    \midrule
    \multirow{9}[2]{*}{ALBEF}& PGD& \cellcolor[rgb]{ .827,  .827,  .827}52.45& \cellcolor[rgb]{ .827,  .827,  .827}58.65& 3.06& 6.79& 8.96& 13.21& 10.34&14.65\\
 & Sep-Attack& \cellcolor[rgb]{ .827,  .827,  .827}65.69& \cellcolor[rgb]{ .827,  .827,  .827}73.95& 17.60& 32.95& 31.17& 45.21& 32.82&45.49\\
    & Co-Attack & \cellcolor[rgb]{ .827,  .827,  .827}97.08  & \cellcolor[rgb]{ .827,  .827,  .827}98.36  & 39.52  & 51.24  & 29.82  & 38.92  & 31.29  & 41.99  \\
      & SGA & \cellcolor[rgb]{ .827,  .827,  .827}99.79  & \cellcolor[rgb]{ .827,  .827,  .827}99.95  & 87.67  & 87.88  & 38.04  & 46.17  & 41.63  & 50.36  \\
      & \textbf{MDCS-SGA} & \cellcolor[rgb]{ .827,  .827,  .827}100.0  & \cellcolor[rgb]{ .827,  .827,  .827}99.98  & 91.78  & 91.24  & 41.35  & 49.71  & 45.08  & 53.93  \\
      & DRA & \cellcolor[rgb]{ .827,  .827,  .827}99.79  & \cellcolor[rgb]{ .827,  .827,  .827}99.91  & 89.78  & 90.52  & 46.63  & 57.28  & 50.32  & 59.11  \\
      & \textbf{MDCS-DRA} & \cellcolor[rgb]{ .827,  .827,  .827}99.90  & \cellcolor[rgb]{ .827,  .827,  .827}99.98  & 93.26  & 92.98  & 49.94  & 59.31  & 55.56  & 62.44  \\
      & SA-AET & \cellcolor[rgb]{ .827,  .827,  .827}99.90  & \cellcolor[rgb]{ .827,  .827,  .827}100.0  & 96.31  & 96.19  & 54.23  & 63.50  & 58.88  & 65.18  \\
      & \textbf{MDCS-SAAET} & \cellcolor[rgb]{ .827,  .827,  .827}99.90  & \cellcolor[rgb]{ .827,  .827,  .827}100.0  & \textbf{96.52 } & \textbf{96.71 } & \textbf{60.25 } & \textbf{67.01 } & \textbf{60.54 } & \textbf{67.89 } \\
    \midrule
    \multirow{9}[2]{*}{TCL}& PGD& 6.15& 10.78& \cellcolor[rgb]{ .827,  .827,  .827}77.87& \cellcolor[rgb]{ .827,  .827,  .827}79.84& 10.18& 16.31& 14.81& 21.11\\
 & Sep-Attack& 20.13& 36.48& \cellcolor[rgb]{ .827,  .827,  .827}84.72& \cellcolor[rgb]{ .827,  .827,  .827}86.10& 31.29& 44.65& 33.33&45.80\\
 & Co-Attack & 49.84  & 60.36  & \cellcolor[rgb]{ .827,  .827,  .827}91.68  & \cellcolor[rgb]{ .827,  .827,  .827}95.48  & 32.64  & 42.69  & 32.06  &47.82  \\
      & SGA & 93.33  & 92.84  & \cellcolor[rgb]{ .827,  .827,  .827}100.0  & \cellcolor[rgb]{ .827,  .827,  .827}100.0  & 37.42  & 46.39  & 42.02  & 51.36  \\
      & \textbf{MDCS-SGA} & 95.10  & 94.69  & \cellcolor[rgb]{ .827,  .827,  .827}100.0  & \cellcolor[rgb]{ .827,  .827,  .827}100.0  & 42.45  & 49.19  & 46.87  & 55.71  \\
      & DRA & 95.31  & 95.35  & \cellcolor[rgb]{ .827,  .827,  .827}100.0  & \cellcolor[rgb]{ .827,  .827,  .827}100.0  & 46.26  & 56.80  & 50.70  & 61.54  \\
      & \textbf{MDCS-DRA} & 96.66  & 96.42  & \cellcolor[rgb]{ .827,  .827,  .827}100.0  & \cellcolor[rgb]{ .827,  .827,  .827}100.0  & 50.92  & 59.76  & 56.19  & 65.01  \\
      & SA-AET & 98.85  & 98.48  & \cellcolor[rgb]{ .827,  .827,  .827}100.0  & \cellcolor[rgb]{ .827,  .827,  .827}99.98  & 53.99  & 63.27  & 59.64  & 68.44  \\
      & \textbf{MDCS-SAAET} & \textbf{99.17 } & \textbf{98.88 } & \cellcolor[rgb]{ .827,  .827,  .827}100.0  & \cellcolor[rgb]{ .827,  .827,  .827}100.0  & \textbf{57.79 } & \textbf{65.98 } & \textbf{62.71 } & \textbf{71.01 } \\
    \midrule
    \multirow{9}[2]{*}{CLIP$_{\text{ViT}}$}& PGD& 3.13& 6.48& 4.85& 8.17& \cellcolor[rgb]{ .827,  .827,  .827}69.33& \cellcolor[rgb]{ .827,  .827,  .827}84.79& 13.03& 17.43\\
 & Sep-Attack& 7.61& 20.58& 10.12& 20.20& \cellcolor[rgb]{ .827,  .827,  .827}76.92& \cellcolor[rgb]{ .827,  .827,  .827}87.44& 29.89&38.32\\
 & Co-Attack & 8.55  & 20.18  & 10.01  & 21.29  & \cellcolor[rgb]{ .827,  .827,  .827}78.5  & \cellcolor[rgb]{ .827,  .827,  .827}87.50  & 29.50  &38.49  \\
      & SGA & 21.58  & 34.89  & 24.66  & 35.83  & \cellcolor[rgb]{ .827,  .827,  .827}100.0  & \cellcolor[rgb]{ .827,  .827,  .827}100.0  & 52.49  & 60.38  \\
      & \textbf{MDCS-SGA} & 29.30  & 40.69  & 30.56  & 39.81  & \cellcolor[rgb]{ .827,  .827,  .827}100.0  & \cellcolor[rgb]{ .827,  .827,  .827}100.0  & 57.60  & 66.00  \\
      & DRA & 27.95  & 43.29  & 29.08  & 44.83  & \cellcolor[rgb]{ .827,  .827,  .827}100.0  & \cellcolor[rgb]{ .827,  .827,  .827}100.0  & 62.45  & 69.47  \\
      & \textbf{MDCS-DRA} & 34.41  & 48.15  & 35.51  & 48.40  & \cellcolor[rgb]{ .827,  .827,  .827}100.0  & \cellcolor[rgb]{ .827,  .827,  .827}100.0  & 68.45  & 73.41  \\
      & SA-AET & 35.97  & 50.28  & 37.93  & 51.36  & \cellcolor[rgb]{ .827,  .827,  .827}100.0  & \cellcolor[rgb]{ .827,  .827,  .827}100.0  & 69.09  & 74.00  \\
      & \textbf{MDCS-SAAET} & \textbf{42.34 } & \textbf{54.86 } & \textbf{43.52 } & \textbf{55.45 } & \cellcolor[rgb]{ .827,  .827,  .827}100.0  & \cellcolor[rgb]{ .827,  .827,  .827}100.0  & \textbf{73.18 } & \textbf{77.53 } \\
    \midrule
    \multirow{9}[2]{*}{CLIP$_{\text{CNN}}$}& PGD& 2.29& 6.15& 4.53& 8.88& 5.41& 12.08& \cellcolor[rgb]{ .827,  .827,  .827}89.78& \cellcolor[rgb]{ .827,  .827,  .827}92.25\\
 & Sep-Attack& 9.38& 22.99& 11.28& 25.45& 26.13& 39.24& \cellcolor[rgb]{ .827,  .827,  .827}93.61&\cellcolor[rgb]{ .827,  .827,  .827}95.31\\
 & Co-Attack & 10.53  & 23.62  & 12.54  & 26.05  & 27.24  & 40.62  & \cellcolor[rgb]{ .827,  .827,  .827}95.91  &\cellcolor[rgb]{ .827,  .827,  .827}96.50  \\
      & SGA & 15.02  & 28.60  & 18.34  & 32.26  & 39.51  & 51.16  & \cellcolor[rgb]{ .827,  .827,  .827}100.0  & \cellcolor[rgb]{ .827,  .827,  .827}100.0  \\
      & \textbf{MDCS-SGA} & 19.19  & 33.44  & 23.92  & 36.48  & 45.89  & 56.22  & \cellcolor[rgb]{ .827,  .827,  .827}99.87  & \cellcolor[rgb]{ .827,  .827,  .827}100.0  \\
      & DRA & 19.08  & 33.96  & 22.02  & 37.45  & 48.34  & 59.02  & \cellcolor[rgb]{ .827,  .827,  .827}100.0  & \cellcolor[rgb]{ .827,  .827,  .827}99.93  \\
      & \textbf{MDCS-DRA} & 23.04  & 39.43  & 25.82  & 41.05  & 55.83  & 63.89  & \cellcolor[rgb]{ .827,  .827,  .827}100.0  & \cellcolor[rgb]{ .827,  .827,  .827}100.0  \\
      & SA-AET & 23.88  & 37.89  & 25.18  & 41.69  & 54.60  & 63.14  & \cellcolor[rgb]{ .827,  .827,  .827}100.0  & \cellcolor[rgb]{ .827,  .827,  .827}99.97  \\
      & \textbf{MDCS-SAAET} & \textbf{30.76 } & \textbf{45.02 } & \textbf{33.40 } & \textbf{47.57 } & \textbf{61.72 } & \textbf{69.56 } & \cellcolor[rgb]{ .827,  .827,  .827}100.0  & \cellcolor[rgb]{ .827,  .827,  .827}99.97  \\
    \bottomrule
    \end{tabular}
\end{table*}

As shown in Tab. \ref{tab:blackbox-vqa}, integrating our MDCS consistently and significantly boosts the success rates of state-of-the-art attacks (SGA, DRA, SA-AET) across all evaluated models. We summarize three key findings: (1) MDCS is particularly effective against CLIP-based models. For instance, in the challenging cross-architecture transfer from $\text{CLIP}_{\text{CNN}}$ to TCL, it elevates the IR R@1 of SA-AET by 8.22\% (from 25.18\% to 33.40\%). (2) The performance gains on TR are are often more pronounced than on IR. When transferring from $\text{CLIP}_{\text{ViT}}$ to ALBEF, MDCS brings a +6.37\% improvement in TR versus +4.58\% in IR. (3) The benefits of MDCS are complementary to  state-of-the-art attacks. While SA-AET is already a strong baseline, MDCS-SAAET further pushes the TR R@1 from 54.23\% to 60.25\% against $\text{CLIP}_{\text{ViT}}$. These results collectively affirm MDCS as a highly effective, model-agnostic, and plug-and-play module for boosting adversarial transferability.

We also provide a visualization of our generated AEs on retrieval tasks in Fig. \ref{fig:vis}. As shown, the generated adversarial perturbations by MDCS-SAAET are imperceptible.

\begin{figure}[htbp]
\centering
\includegraphics[width=0.48\textwidth]{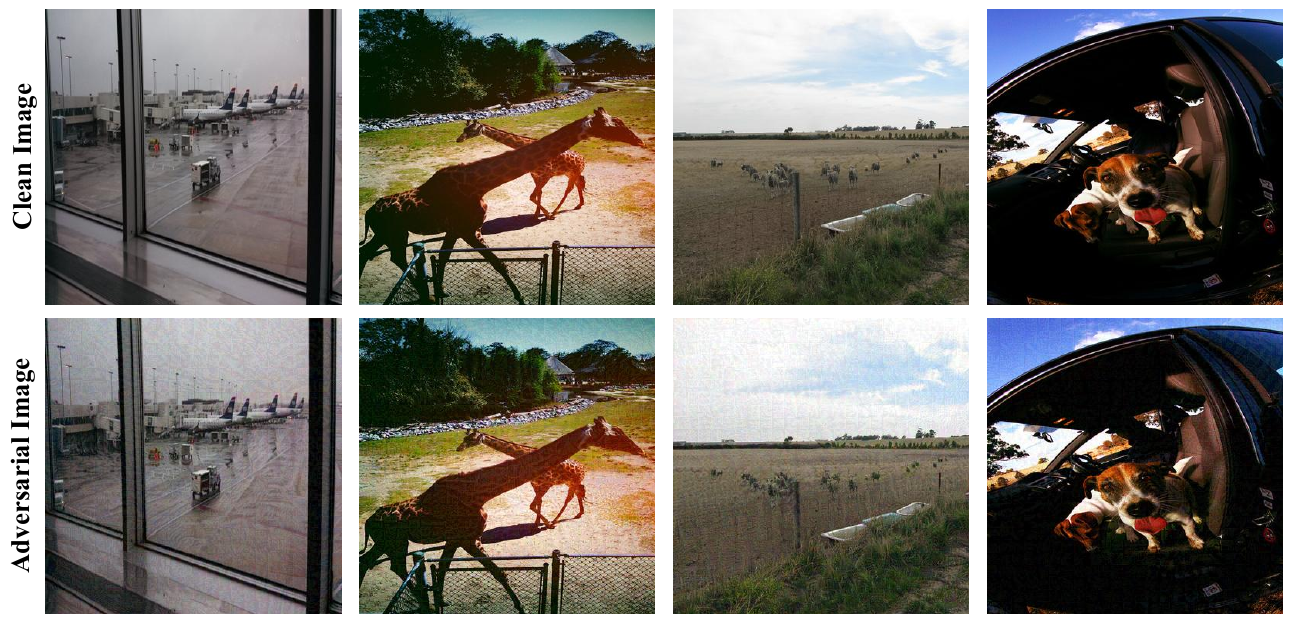}
\caption{Visualization of AEs on image-text retrieval tasks with MDCS-SAAET. AEs are crafted on ALBEF and are constrained by an $L_{\infty}$ norm of 8/255, generated over 10 iterations with a step size of 2/255.}
\label{fig:vis}
\end{figure}

\section{Conclusion}

This paper addresses the issues of non-convergence and instability in sign-based adversarial attacks. From an optimization perspective, we develop a series of novel optimizers that incorporate the MDCS strategy to improve both stability and transferability. Specifically, we prove that MI-FGSM equipped with MDCS achieves the optimal convergence rate. Experimental results on image classification and cross-modal retrieval tasks demonstrate that MDCS serves as a universal strategy, seamlessly integrating with existing sign-based attacks and significantly enhancing their transferability and stability.

\textbf{Acknowledgments.} This work was supported in part by the National Natural Science Foundation of China (62576351) and the China Postdoctoral Science Foundation (2024M764294).

{
    \small
    \bibliographystyle{ieeenat_fullname}
    \bibliography{main}
}
\clearpage
\setcounter{page}{1}
\maketitlesupplementary

\section{Convergence Analysis of MDCS-MI}
\label{app:appendix}

\begin{lemma}
\label{lem:lemma3.2}
\textup{\cite{mukkamala2017variants}}   $\forall \bm{x}_1, \bm{x}_2 \in \mathbb{R}^{d}$, we have
\begin{equation}
\|P_{\mathbf{Q},\bm{D}_{t}^{-1}}(\boldsymbol{x}_1)-P_{\mathbf{Q},\bm{D}_{t}^{-1}}(\boldsymbol{x}_2) \|_{\bm{D}_{t}^{-1}} \leq \| \boldsymbol{x}_1 - \boldsymbol{x}_2 \|_{\bm{D}_{t}^{-1}}.
\end{equation}
and 
\begin{equation}
\bm{x}^{*} = P_{\mathbf{Q},\bm{D}_{t}^{-1}}(\bm{x}^{*}).
\end{equation}
\end{lemma}

\begin{lemma}
\label{lem:usefullemma1}
Let $1>\lambda>0$ and $\beta>0$. Let $\beta_t=\beta \lambda^{t-1}$.
Suppose $\{\bm{m}_{t+1}\}_{t=1}^{\infty}$ is generated by MDCS-MI. Then there exits a $M_2>1$ such that
$$\|\bm{m
}_{t+1}\| \leq M_2,\ \forall t>0,$$
\end{lemma}
{\it Proof} \ Note
$$\| \nabla_{\boldsymbol{x}}J(\boldsymbol{x}^{adv})\| \leq \| \nabla_{\boldsymbol{x}}J(\boldsymbol{x}^{adv})\|_{1}.$$ 
Then
$$ \|\bm{m}_{t+1}\|\leq \beta_t \|\bm{m}_{t}\|+1  \leq  \beta_{t-1} \|\bm{m}_{t-1}\|+ \beta_{t}+1 \leq \sum_{i=1}^{t} \beta_i +1.$$
Let $M_2=\sum_{i=1}^{\infty} \beta_i+1$, Lemma \ref{lem:usefullemma1} follows from the convergence of $\sum_{i=1}^{t} \beta_i$.\\

Note $$d_{t,i} = \mathrm{min} (\frac{1}{|\bm{m}_{t+1,i}| }, d_{t-1,i} ).$$
According to Lemma \ref{lem:usefullemma1}, we can get
\begin{lemma}
\label{lem:usefullemma2}
Let $1>\lambda>0$ and $\beta>0$. Let $\beta_t=\beta \lambda^{t-1}$.
Suppose $\{\bm{D}_{t}\}_{t=1}^{\infty}$ is generated by MDCS-MI. Then for each $t$ and $i$,
$$\frac{1}{M_2}\leq d_{t,i}\leq 1.$$
\end{lemma}

{\it Proof of Theorem \ref{thm:bigtheorem}}

From Lemma \ref{lem:lemma3.2}, we know
\begin{equation*}
   \begin{aligned}
   & \|\bm{x}_{t+1}^{adv}-\boldsymbol{x}^{\ast}\|^2_{\bm{D}_{t}^{-1}} \\
    &= \|P_{\mathbf{Q},\bm{D}_{t}^{-1}} (\bm{x}_{t}^{adv}+\alpha_t {\bm{D}_{t}} \boldsymbol{m}_{t+1}) -P_{\mathbf{Q},\bm{D}_{t}^{-1}}(\boldsymbol{x}^{\ast})\|^2_{\bm{D}_{t}^{-1}} \\
     &\leq \|\bm{x}_{t}^{adv}+\alpha_t {\bm{D}_{t}} \boldsymbol{m}_{t+1} -\boldsymbol{x}^{\ast}\|^2_{\bm{D}_{t}^{-1}} \\
   & = \|\bm{x}_{t}^{adv}-\boldsymbol{x}^{\ast}\|^2_{\bm{D}_{t}^{-1}} + \| \alpha_t {\bm{D}_{t}} \boldsymbol{m}_{t+1} \|^2_{\bm{D}_{t}^{-1}} \\
   & \ \ \ \ +2\alpha_t \langle \boldsymbol{m}_{t+1} , \bm{x}_{t}^{adv}-\boldsymbol{x}^{\ast}\rangle \\
   & = \|\bm{x}_{t}^{adv}-\boldsymbol{x}^{\ast}\|^2_{\bm{D}_{t}^{-1}} + \| \alpha_t {\bm{D}_{t}} \boldsymbol{m}_{t+1} \|^2_{\bm{D}_{t}^{-1}}  \\
   & \ \ \ \ + 2\alpha_t\langle \beta_t \boldsymbol{m}_{t}+ \displaystyle\frac{\nabla_{\bm{x}} J(\bm{x}_{t}^{adv})}{\|\nabla_{\bm{x}} J(\bm{x}_{t}^{adv})\|_1}, \bm{x}_{t}^{adv}-\boldsymbol{x}^{\ast}\rangle.
   \end{aligned}
 \end{equation*}
 Rearrange the inequality, we have
  \begin{equation*}
  \begin{aligned}
   &\frac{2\alpha_t}{\|\nabla_{\bm{x}} J(\bm{x}_{t}^{adv})\|_1} \langle\nabla_{\bm{x}} J(\bm{x}_{t}^{adv}), \bm{x}_{t}^{adv}-\boldsymbol{x}^{\ast}\rangle \\
   & \geq  \|\bm{x}_{t+1}^{adv}-\boldsymbol{x}^{\ast}\|^2_{\bm{D}_{t}^{-1}} -\|\bm{x}_{t}^{adv}-\boldsymbol{x}^{\ast}\|^2_{\bm{D}_{t}^{-1}} \\
   & \ \ \ \ - \| \alpha_t {\bm{D}_{t}} \boldsymbol{m}_{t+1} \|^2_{\bm{D}_{t}^{-1}} 
   - 2\alpha_t \beta_t \langle\boldsymbol{m}_{t} , \bm{x}_{t}^{adv}-\boldsymbol{x}^{\ast}\rangle,
   \end{aligned}
 \end{equation*}
 i.e.,
  \begin{equation*}
  \begin{aligned}
 & \frac{1}{\|\nabla_{\bm{x}} J(\bm{x}_{t}^{adv})\|_1} \langle\nabla_{\bm{x}} J(\bm{x}_{t}^{adv}), \bm{x}_{t}^{adv}-\boldsymbol{x}^{\ast}\rangle\\
 & \geq  \frac{\|\bm{x}_{t+1}^{adv}-\boldsymbol{x}^{\ast}\|^2_{\bm{D}_{t}^{-1}} -\|\bm{x}_{t}^{adv}-\boldsymbol{x}^{\ast}\|^2_{\bm{D}_{t}^{-1}}}{2\alpha_t } - \frac{\alpha_t \|\boldsymbol{m}_{t+1} \|^2}{2} \\
 & \ \ \ \ -\beta_t \langle\boldsymbol{m}_{t} , \bm{x}_{t}^{adv}-\boldsymbol{x}^{\ast}\rangle \\
 & \geq  \frac{\|\bm{x}_{t+1}^{adv}-\boldsymbol{x}^{\ast}\|^2_{\bm{D}_{t}^{-1}} -\|\bm{x}_{t}^{adv}-\boldsymbol{x}^{\ast}\|^2_{\bm{D}_{t}^{-1}}}{2\alpha_t } - \frac{\alpha_t \|\boldsymbol{m}_{t+1} \|^2}{2}  \\
 & \ \ \ \ - \frac{\beta_t \alpha_t \| \boldsymbol{m}_{t} \|^2}{2 } - \frac{ \beta_t \|\bm{x}_{t}^{adv}-\boldsymbol{x}^{\ast}\|^2}{2 \alpha_t}. \\
 \end{aligned}
 \end{equation*}
 Using the property of concave functions,
 \begin{equation*}
   \langle\nabla_{\bm{x}} J(\bm{x}_{t}^{adv}), \bm{x}_{t}^{adv}-\boldsymbol{x}^{\ast}\rangle
  \leq J(\boldsymbol{x}^{adv}_{t})-J(\boldsymbol{x}^{\ast}).
 \end{equation*}
 Then
\begin{equation*}
\begin{aligned}
& \frac{J(\boldsymbol{x}^{\ast})-J(\boldsymbol{x}^{adv}_{t})}{M} \\
& \leq \frac{\|\bm{x}_{t}^{adv}-\boldsymbol{x}^{\ast}\|^2_{\bm{D}_{t}^{-1}}- \|\bm{x}_{t+1}^{adv}-\boldsymbol{x}^{\ast}\|^2_{\bm{D}_{t}^{-1}}} {2\alpha_t } \\
& \ \ \ \ + \frac{\alpha_t \|\boldsymbol{m}_{t+1} \|^2}{2} + \frac{\beta_t \alpha_t\| \boldsymbol{m_{t}} \|^2}{2 } + \frac{ \beta_t \|\bm{x}_{t}^{adv}-\boldsymbol{x}^{\ast}\|^2}{2 \alpha_t}. \\
\end{aligned}
\end{equation*}
Summing this inequality from $t = 1$ to $T$, we obtain
\begin{equation*}
\begin{aligned}
& \frac{1}{M}\sum_{t=1}^{T}\left[J(\boldsymbol{x}^{\ast})-J(\boldsymbol{x}^{adv}_{t})\right] \\
&  \leq\underbrace{\sum_{t=1}^{T}\frac{\|\bm{x}_{t}^{adv}-\boldsymbol{x}^{\ast}\|^2_{\bm{D}_{t}^{-1}} - \|\bm{x}_{t+1}^{adv}-\boldsymbol{x}^{\ast}\|^2_{\bm{D}_{t}^{-1}} }{2\alpha_t }}_{P_1} \\
& \ \ \ \ +\underbrace{\sum_{t=1}^{T}\frac{ \beta_t \|\bm{x}_{t}^{adv}-\boldsymbol{x}^{\ast}\|^2}{2 \alpha_t}}_{P_2}\\
&\ \ \ \ +\underbrace{\sum_{t=1}^{T}\frac{\alpha_t \|\boldsymbol{m}_{t+1} \|^2}{2}+\sum_{t=1}^{T}\frac{\beta_t \alpha_t\| \boldsymbol{m}_{t} \|^2}{2 }.}_{P_3}
\end{aligned}
\end{equation*}

To bound $P_1$, we have
\begin{equation}\label{c1}
\begin{aligned}
     P_1=&\sum_{t=1}^{T}\frac{\|\bm{x}_{t}^{adv}-\boldsymbol{x}^{\ast}\|^2_{\bm{D}_{t}^{-1}} - \|\bm{x}_{t+1}^{adv}-\boldsymbol{x}^{\ast}\|^2_{\bm{D}_{t}^{-1}} }{2\alpha_t }\\
     =&\sum_{t=2}^{T}\left(\frac{\|\boldsymbol{x}_{t}^{adv}-\boldsymbol{x}^{\ast}\|^2_{\bm{D}_{t}^{-1}}}{2\alpha_t}-\frac{\|\boldsymbol{x}_{t}^{adv}-\boldsymbol{x}^{\ast}\|^2_{\bm{D}_{t-1}^{-1}}}{2\alpha_{t-1}}\right) \\ 
     & + \frac{\|\boldsymbol{x}_{1}^{adv}-\boldsymbol{x}^{\ast}\|^2_{\bm{D}_{1}^{-1}}}{2\alpha_1}
     -\frac{\|\boldsymbol{x}_{T+1}^{adv}-\boldsymbol{x}^{\ast}\|^2_{\bm{D}_{T}^{-1}}}{2\alpha_T }\\
    \end{aligned}
\end{equation}
According to Lemma \ref{lem:usefullemma2},
\begin{equation}\label{c}
\begin{aligned}
     P_1
     \leq & \sum_{t=2}^{T}\sum_{i=1}^{d}\left(\frac{1}{2\alpha_t d_{t,i}}-\frac{1}{2\alpha_{t-1}d_{t-1,i}}\right)|\boldsymbol{x}_{t,i}^{adv}-\boldsymbol{x_i}^{\ast}|^2 \\ 
     & + \frac{\|\boldsymbol{x}_{1}^{adv}-\boldsymbol{x}^{\ast}\|^2_{\bm{D}_{1}^{-1}}}{2\alpha_1} \\
     \leq & \sum_{i=1}^{d}\frac{1}{2\alpha_T d_{T,i}}G^2\\
     \leq & \frac{d M_2 G^2\sqrt{T}}{2\gamma }.
\end{aligned}
\end{equation}

To bound $P_2$, we have
\begin{equation}\label{d}
\begin{aligned}
P_2 & =\sum_{t=1}^{T}\frac{ \beta_t \|\bm{x}_{t}^{adv}-\boldsymbol{x}^{\ast}\|^2}{2 \alpha_t} \\
    & \leq\frac{{\beta G^2}}{2\gamma}\sum_{t=1}^{T} \lambda^{t-1} \sqrt{t} \\
    & \leq \frac{{\beta G^2}}{2\gamma}\sum_{t=1}^{T} \lambda^{t-1} t \\
    & \leq \frac{{\beta G^2}}{2\gamma (1-\lambda)^2}
\end{aligned}
\end{equation}

To bound $P_3$, according to Lemma \ref{lem:lemma3.2}, we have
  \begin{equation}\label{e}
   \begin{aligned}
     P_3&=\sum_{t=1}^{T}\frac{\alpha_t \|\boldsymbol{m}_{t+1} \|^2}{2}+\sum_{t=1}^{T}\frac{\beta_t \alpha_t\| \boldsymbol{m}_{t} \|^2}{2 } \\
     & \leq \sum_{t=1}^{T} \frac{\alpha_t {M
     _{2}}^2}{2}+\sum_{t=1}^{T}\frac{\beta_t \alpha_t {M_{2}}^2}{2 }\\
     &\leq \frac{M_{2}^2}{2} \sum_{t=1}^{T} \frac{\gamma}{\sqrt t} +\frac{D_{2}^2}{2} \sum_{t=1}^{T}\frac{\beta_t \gamma}{\sqrt t} \\
     & \leq 2\gamma {M_{2}}^2 \sqrt{T}.
   \end{aligned}
\end{equation}
Combining (\ref{c}), (\ref{d}) and (\ref{e}), we have
\begin{equation*}
\begin{aligned}
    & \frac{1}{M} \sum_{t=1}^{T}\left( J(\boldsymbol{x}^{\ast})-J(\boldsymbol{x}^{adv}_{t})\right) \\
    & \leq \frac{d M_2 G^2\sqrt{T}}{2\gamma } + \frac{\beta G^2}{2\gamma (1-\lambda)^2} +   2\gamma {M_{2}}^2 \sqrt{T}.
   \end{aligned}
\end{equation*}
Thus
\begin{equation*}
\begin{aligned}
    & \frac{1}{MT}\sum_{t=1}^{T}\left( J(\boldsymbol{x}^{\ast})-J(\boldsymbol{x}^{adv}_{t}) \right) \\
   &  \leq \frac{d M_2 G^2}{2\gamma \sqrt{T}} +  \frac{\beta G^2}{2\gamma (1-\lambda)^2} +\frac{   2\gamma {M_{2}}^2}{ \sqrt{T}}.
   \end{aligned}
 \end{equation*}
 By concavity of $J(\boldsymbol{x})$, we obtain
 \begin{equation}\label{f}
   \begin{aligned}
   &  J(\boldsymbol{x}^{\ast})-J(\boldsymbol{\bar{x}}^{adv}_T) \\
   &  \leq M \left(\frac{d M_2 G^2}{2\gamma \sqrt{T}} +  \frac{\beta G^2}{2\gamma (1-\lambda)^2} +\frac{   2\gamma {M_{2}}^2}{ \sqrt{T}}\right).
   \end{aligned}
 \end{equation}

This completes the proof of Theorem \ref{thm:bigtheorem}.\\

\section{Pseudocode for MDCS-MEF and MDCS-OPS }

For clarity,  we give the detailed descriptions of MDCS-MEF (\cref{alg:MEF}) and MDCS-OPS (\cref{alg:OPS}) for image classification tasks.

\begin{algorithm}
\caption{MDCS-MEF}
\label{alg:MEF}
    \begin{algorithmic}[1]
    \Require
    Loss function $J$, a raw example $\bm{x}$ with ground-truth label $y$, the perturbation size $\epsilon$, momentum parameter $0 < \beta_t \leq 1$, step-size $\alpha_t > 0$, maximum iteration $T$, the outer/inner decay factor $\mu_{outer}$/$\mu_{inner}$; the number of randomly sampled examples, $N$; neighborhood radius $\xi$, exploration radius $\gamma$.
    \State Initialize $\bm{x}^{adv}_0=\bm{x}$, $d_{0,i} = 1$. $\bm{g}^{outer}_{0} = 0$; $\{g^{inner}_{0,n}\}_{n=1}^N = 0$; $\bm{x}^{adv}_{0}=\bm{x}$; $\alpha={\epsilon/}{T}$
    \Repeat
    \State ${\{\bm{x}_{t,n}\}}_{n=1}^N \in U_{\gamma}(\bm{x}^{adv}_{t})$
    \State ${\{\bm{x}^{'}_{t,n}\}}_{n=1}^N = {\{\bm{x}_{t,n}\}}_{n=1}^N + \xi \cdot \mathrm{sign}(\ g^{inner}_{t,n}) \}_{n=1}^N)$
    \State $\{g_{t,n}\}_{n=1}^N = \nabla_{\bm{x}}J({\{\bm{x}^{'}_{t,i}\}}_{i=1}^N)$
    \State $\{g^{inner}_{t,n}\}_{n=1}^N = \frac{\{g_{t,n}\}_{n=1}^N}{\|\{g_{t,n}\}_{n=1}^N\|_{1}} - \mu_{inner}\{g^{inner}_{t,n}\}_{n=1}^N$
    \State $\bm{g}^{outer}_{t} = \mu_{outer}\ \bm{g}^{outer}_{t} + \frac{1}{N}\sum_{i=1}^{N}{\frac{\{g_{t,n}\}_{n=1}^N}{\|\{g_{t,n}\}_{n=1}^N\|_{1}}}$
    \State $d_{t,i} = \mathrm{min} (\frac{1}{|\bm{g}^{outer}_{t,i}| }, d_{t-1,i} )$, 
    \State $\bm{D}_t = \mathrm{diag}(\bm{d}_{t})$,
    \State $\bm{x}^{adv}_{t+1} = P_{\mathbf{Q}, \bm{D}^{-1}_t} (\bm{x}^{adv}_{t} + \alpha_t \cdot \bm{D}_t \cdot g_{t}^{outer})$.
    \Until {$t = T$}
    \Ensure  $\bm{x}_{T}^{adv}$.
    \end{algorithmic}
\end{algorithm}

\begin{algorithm}
\caption{MDCS-OPS}
\label{alg:OPS}
\begin{algorithmic}[1]
\Require 
Loss function $J$, a raw example $\bm{x}$ with ground-truth label $y$, the perturbation size $\epsilon$, the maximum iteration $T$ and decay factor $\mu$; Level list $K_{\text{list}}$; Radius list $R_{\text{list}}$; Number of operator samples $N_p$; Number of perturbation samples $N_e$.
\State Initialize $\bm{x}^{adv}_0=\bm{x}$, $d_{0,i} = 1$, $\alpha = \epsilon / T, \bm{m}_0 = \bm{0}, \Delta_0 = 0.$
    \Repeat
    \State $\bar{g} = \nabla_{\bm{x}} {J}(f(\bm{x} + \Delta_i), y).$
    \For{$\delta$ in $S(D, N_e)$}
        \For{$op$ in $S(P, N_p)$}
        \State $\bar{\bm{g}} = \bar{\bm{g}} + \nabla_{\bm{x}} J \left[f(op(\bm{x} + \Delta_i + \delta)), y\right]$
        \EndFor
    \EndFor
    \State $\bar{\bm{g}} = \frac{\bar{\bm{g}}}{N_e \times N_p + 1}.$
    \State $\bm{m}_{t+1} = \mu \cdot \bm{m}_t + \frac{\bar{\bm{g}}}{\|\bar{\bm{g}}\|_1}.$
    \State $d_{t,i} = \mathrm{min} (\frac{1}{|\bm{m}_{t,i}| }, d_{t-1,i} )$,
    \State $\bm{D}_t = \mathrm{diag}(\bm{d}_{t})$,
    \State $\Delta_{i+1} = P_{\mathbf{Q}, \bm{D}^{-1}_t}(\Delta_i + \alpha_t \cdot \bm{D}_t \cdot \bm{m}_{t+1}).$
    \Until {$t = T$}
    \Ensure  $\bm{x}_{T}^{adv}$.
\end{algorithmic}
\end{algorithm}

\section{Additional Experiments}

This section presents more experimental results to provide a comprehensive empirical evaluation of the proposed methods on image classification, Visual Question Answering (VQA), and cross-modal retrieval tasks.

\subsection{Image Classification Tasks}
\label{app:cv}

\subsubsection{Transferability and Stability}
\label{app:8.1.1}

In this subsection, we extend our evaluation to include additional source models, namely VGG16 and ViT-B/16. The success rates of these transfer attacks are quantified in Tab. \ref{tab:addition_normal}. The results clearly demonstrate that our derived MDCS strategy consistently enhances adversarial transferability compared to the baseline attacks. Notably, the MDCS-OPS variant achieves the highest attack success rate in the black-box setting, establishing a new state-of-the-art. As shown in Fig. \ref{fig:app_stable}, integrating the MDCS strategy into TI and SI significantly enhances the iterative stability of adversarial attacks.

\subsubsection{Convergence Behavior}
\label{app:8.1.2}

To make a through comparison, we also investigate the white-box and black-box attack convergence behavior of loss function $J(\bm{x}_{t}^{adv},y)$ with respect to the number of iterations. The relationship between the value of loss function $J(\bm{x}_{t}^{adv},y)$ and the number of iterations is shown in Fig. \ref{fig:bb_loss}. We summarize three key findings: (1) Overfitting is observed in I-FGSM and MI-FGSM. (2) The stability of black-box ASR ensures that whenever the iteration is stopped, relatively good and reliable transferability can be provided. (3) From an optimization perspective,  Fig.\ref{fig:bb_loss} clearly verify the overfitting of I-FGSM and MI-FGSM, and the stability of our MDCS.  Naturally, if the concerned momentum is already stable like that in OPS, gains from MDCS are moderate.

To sum up, stability refers to maintaining performance across iterations without degradation. The degradation in black-box ASR typically indicates overfitting. Therefore, stability contributes to improved transferability.

\begin{figure*}[htbp]
\setlength{\belowcaptionskip}{-0.2cm}
\centering
\includegraphics[width=0.48\textwidth]{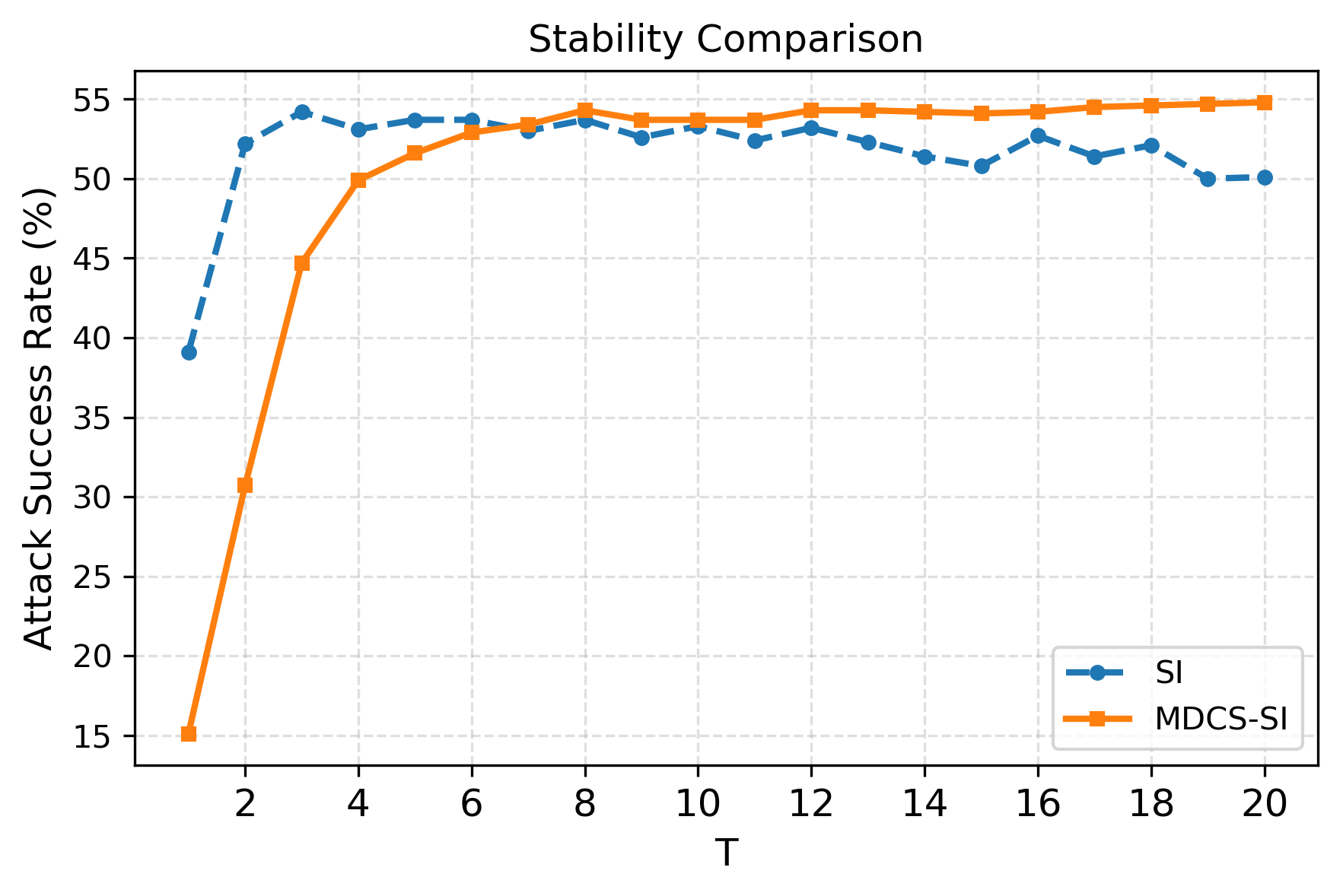}
\includegraphics[width=0.48\textwidth]{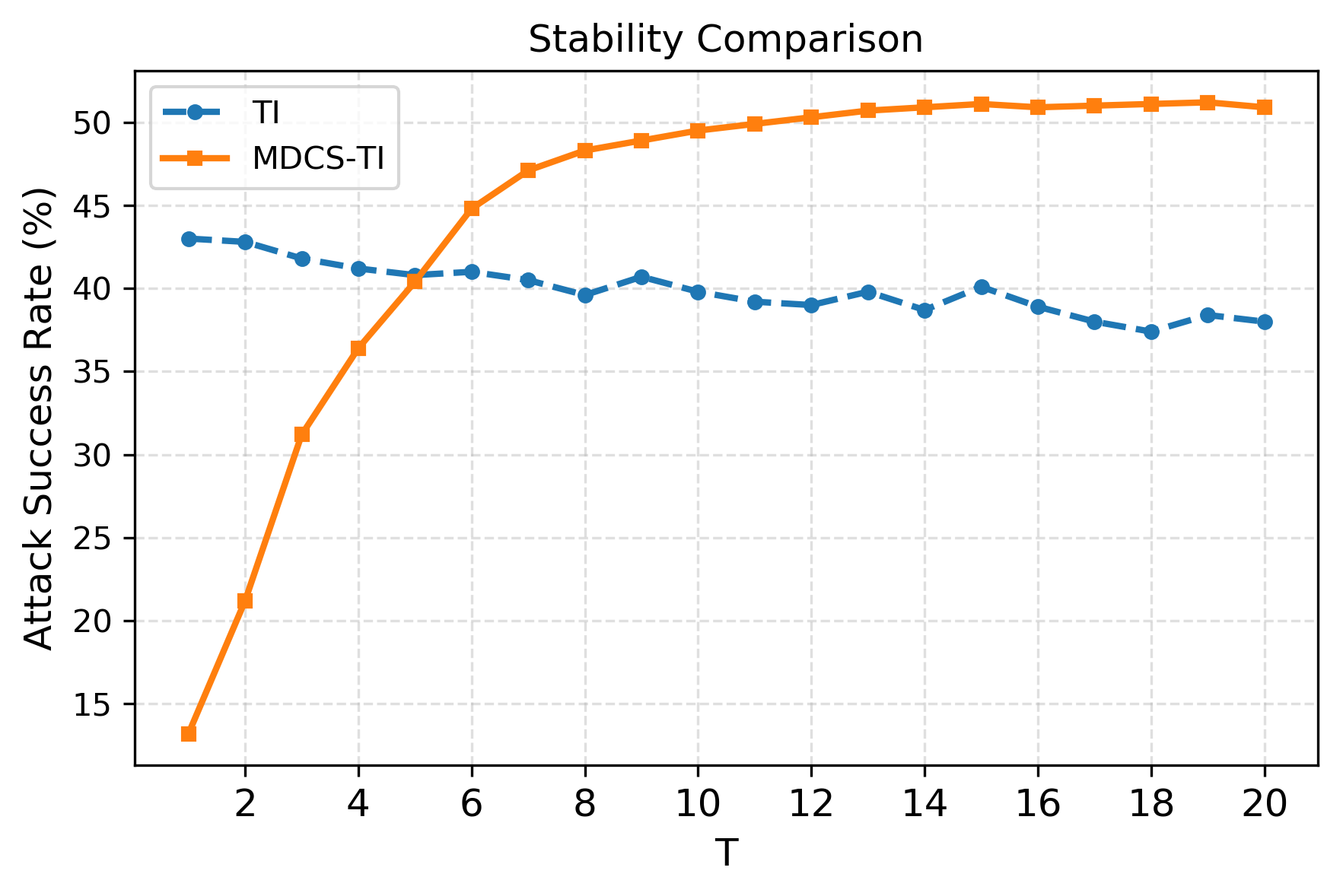}
\caption{Stability comparison of SI and TI (based on MI-FGSM). We employ Res50 as surrogate model and Inc-v3 as target model with $\epsilon = 16/255$.}
\label{fig:app_stable}
\end{figure*}

\begin{figure*}[htbp]
\setlength{\belowcaptionskip}{-0.2cm}
\centering
\includegraphics[width=\textwidth]{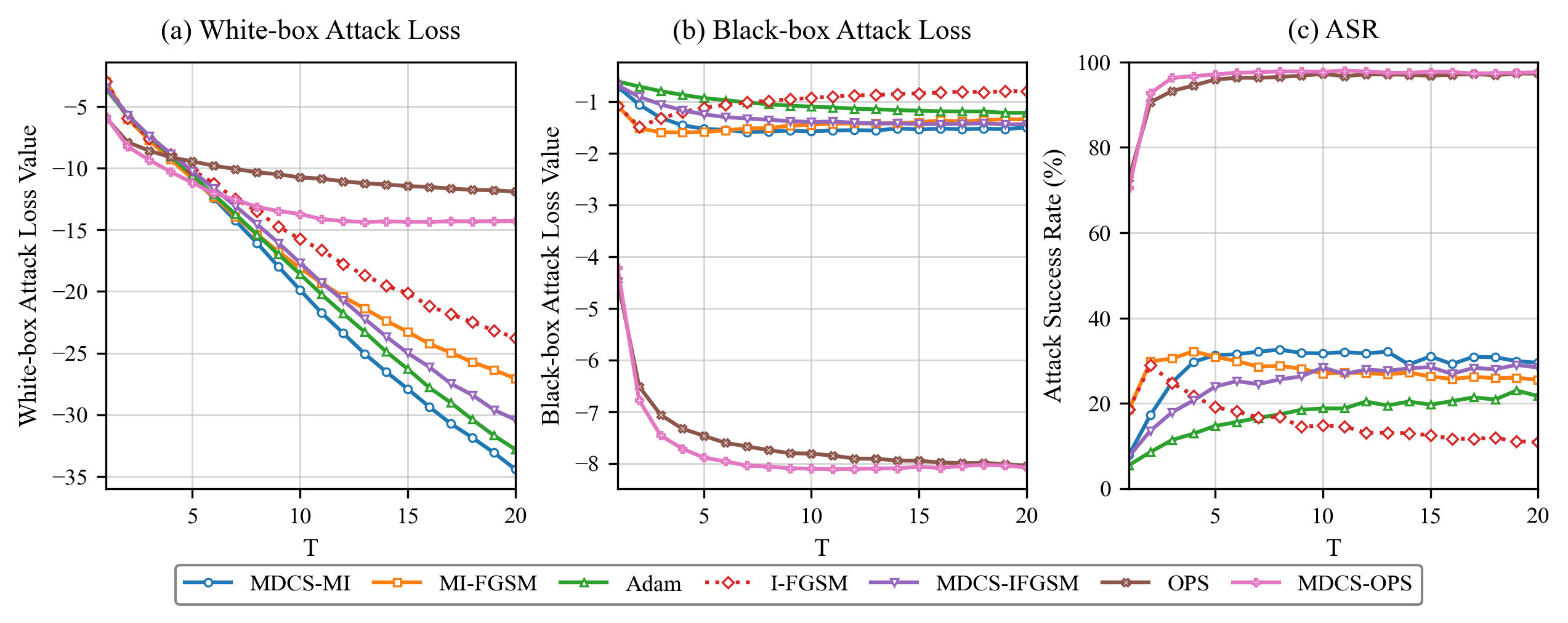}
\caption{Dynamics of white-box and black-box losses. We employ Res50 as surrogate model and Vis-S as target model.}
\label{fig:bb_loss}
\end{figure*}

\begin{figure*}[htbp]
\centering
\includegraphics[width=0.9\textwidth]{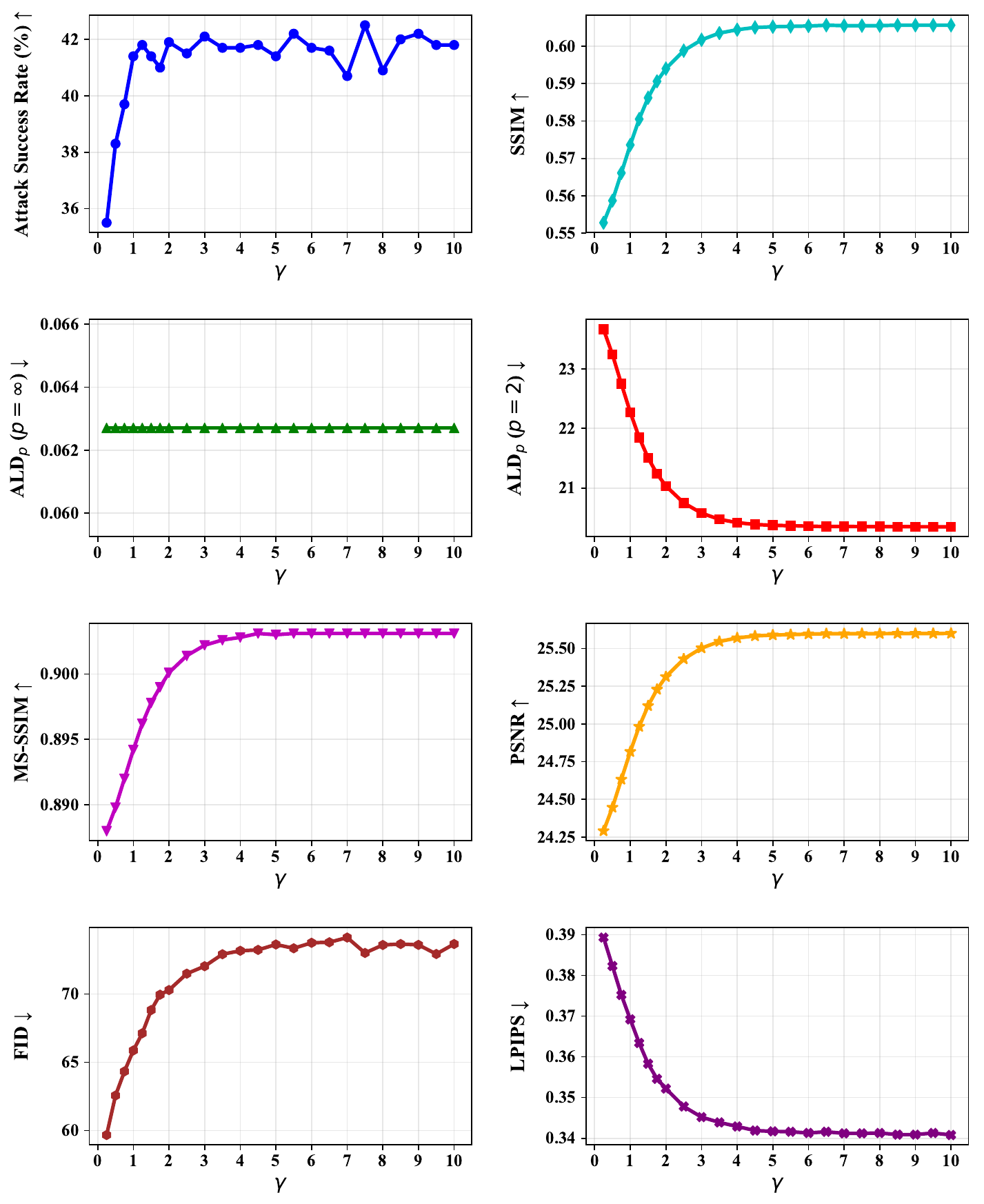}
\caption{Ablation study on the value of $\gamma$ for MDCS-MI. We employ Res50 as surrogate model and Inc-v3 as target model.}
\label{fig:gamma}
\end{figure*}

\subsubsection{Ablation Study}
\label{app:8.1.3}

Here, we investigate the impact of the step-size parameter $\gamma$ on the performance of MDCS-MI, which governs the trade-off between transferability and imperceptibility. According to \cite{zhao2025revisiting}, there remains a notable lack of consensus and attention regarding the proper metrics for evaluating imperceptibility. 
A popular proxy measure of imperceptibility is the $L_p$-norm of the perturbation. It offers a good trade-off between simplicity, mathematical tractability, and practicality in applications \cite{Chawin2024}. Recent research \cite{zhao2025revisiting,zhao2020towards} suggests that it is problematic to solely constrain all attacks with the same $L_{\infty}$ norm bound without more comprehensive comparisons. Thus, we employ 7 different metrics as the indicators of imperceptibility of the crafted AEs, including the Average $L_{\infty}$ Distortion (ALD$_p$), Average $L_{2}$ Distortion (ALD$_p$) \cite{Lin2020NesterovAG}, Peak Signal-to-Noise Ratio (PSNR), Structural Similarity Index Measure (SSIM), Multi Scale Structural Similarity Index Measure (MS-SSIM), Frechet Inception Distance (FID) \cite{Martin2017FID}, Learned Perceptual Image Patch Similarity (LPIPS) \cite{zhang2018lpips}. As shown in Fig. \ref{fig:gamma}, the attack successful rate of MDCS-MI increases with the growth of $\gamma$ (varied within the range of 0.25 to 10). The rate plateaus at around 47\% when $\gamma$ is greater than 2. For imperceptibility, the increase of $\gamma$ leads to a significant improvement in all metrics except for FID and ALD$_p (p=\infty)$. Therefore, a holistic assessment with diverse metrics is essential for evaluating the imperceptibility of AEs. By balancing the transferability and imperceptibility in our experiment, $\gamma$ is determined by simple grid search over the range [2, 4].

\begin{figure}[htbp]
\centering
\includegraphics[width=\linewidth]{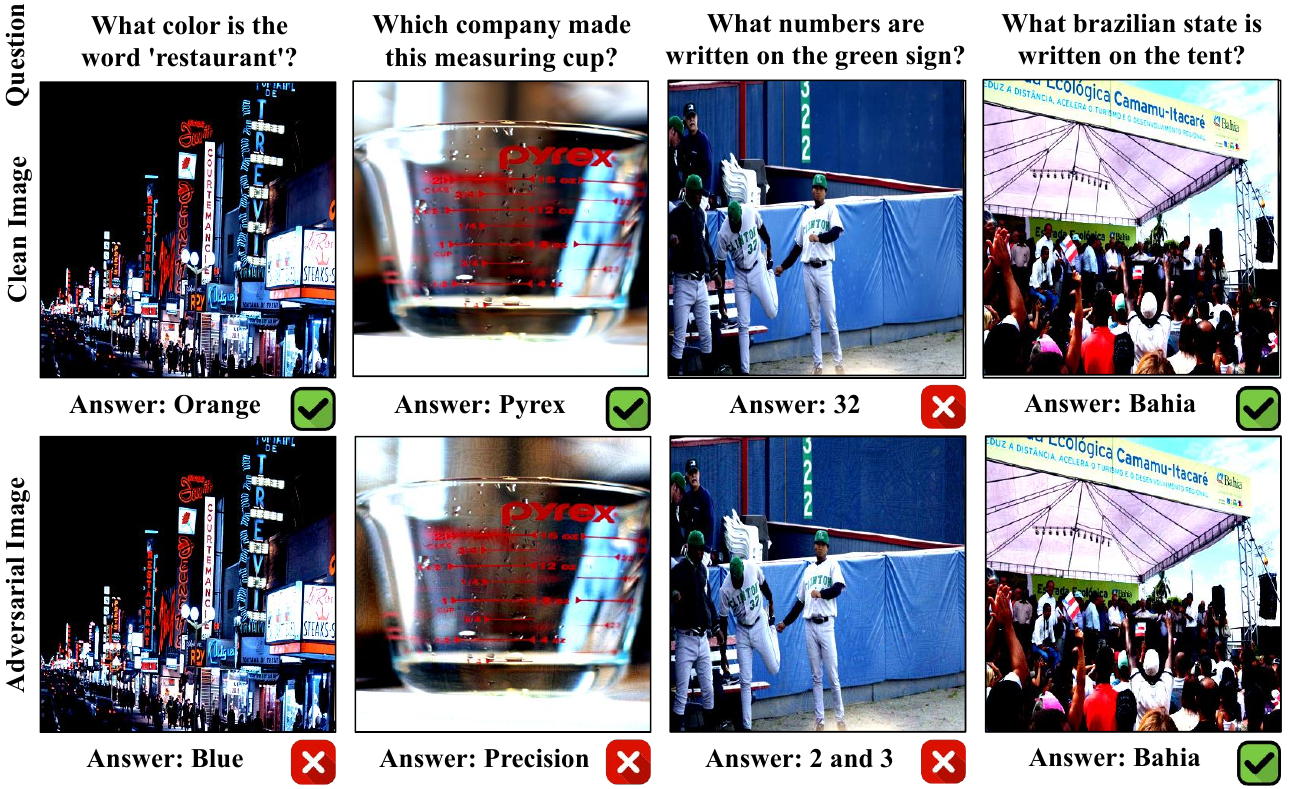}
\caption{A comparison between responses from the clean image and the adversarial image generated by MDCS-MI against prismVLM. Image and questions source are from textVQA.}
\label{fig:vqa_vis}
\end{figure}

\begin{table*}[htbp]\small
\captionsetup{font={normalsize}}
\caption{\normalsize{Transferability comparisons across different networks. AEs are crafted for VGG16 and ViT-B/16. The best results are marked in bold. The gray area represents adversarial attacks under a white-box setting, the rest are black-box attacks.}}
\label{tab:addition_normal}
\centering
\begin{tabular}{c|c|ccccccc}%
\toprule %[2pt]
 Model&Attack&Res50& VGG16 & Mob-v2 & Inc-v3 & ViT-B & PiT-B& Vis-S \\
\midrule %[2pt]
\multirow{12}{*}{VGG16}&I-FGSM& 35.8& \cellcolor[rgb]{ .827,.827,.827}99.9&62.9 &27.1&  8.0& 15.5& 25.1\\
&MI& 59.1& \cellcolor[rgb]{ .827,.827,.827}99.8& 80.2 &55.9& 15.6& 29.6& 44.4\\
&\bf{MDCS-MI}&  59.5&  \cellcolor[rgb]{ .827,.827,.827}99.8& 81.0&59.9& 17.5& 30.3&45.9 \\
&VMI& 74.1&  \cellcolor[rgb]{ .827,.827,.827}99.8& 88.6&72.9& 26.1& 44.2&60.9\\
& EMI& 74.6& \cellcolor[rgb]{ .827,.827,.827}100.0& 90.5& 69.8& 21.5& 38.9&57.9\\
&GRA& 71.3& \cellcolor[rgb]{ .827,.827,.827}100.0&  89.9&79.6 & 27.0 & 42.7& 58.0\\
&MUMODIG&90.6& \cellcolor[rgb]{ .827,.827,.827}100.0&  96.4&90.2& 31.3& 54.9& 78.3\\
&PGN& 72.1& \cellcolor[rgb]{ .827,.827,.827}100.0& 89.4&80.8& 28.0& 44.9& 58.6\\
&MEF& 83.0& \cellcolor[rgb]{ .827,.827,.827}99.9& 93.9&82.8&31.9&51.6&70.3\\
&\bf{MDCS-MEF}& 85.7&\cellcolor[rgb]{ .827,.827,.827} 99.9& 93.5& 83.9& 34.8& 55.8& 73.3\\
&SIA& 92.9& \cellcolor[rgb]{ .827,.827,.827}100.0& 98.1&  86.9&  35.5&  61.9&  84.8\\
&OPS& 97.4& \cellcolor[rgb]{ .827,.827,.827}100.0& 99.2& 99.1& 63.7& 80.1& 93.4\\
&\bf{MDCS-OPS}& \bf{98.1}&\cellcolor[rgb]{ .827,.827,.827}100.0&\bf{99.6}&\bf{99.3}& \bf{66.7}& \bf{82.3}& \bf{93.9}\\
\hline
\multirow{12}{*}{ViT-B/16}&I-FGSM& 18.0&37.8&35.5&25.4 &  \cellcolor[rgb]{ .827,.827,.827}99.9& 29.0& 28.6\\
&MI& 41.2& 63.9& 58.7&47.3& \cellcolor[rgb]{ .827,.827,.827}100.0& 50.2& 51.0\\
&\bf{MDCS-MI}& 45.7&  70.8& 61.3&53.0 & \cellcolor[rgb]{ .827,.827,.827}100.0& 54.0&55.1\\
&VMI& 54.6&  68.9 & 64.7&56.2& \cellcolor[rgb]{ .827,.827,.827}100.0& 67.5&67.3\\
& EMI& 59.3& 75.9 & 71.8& 64.7& \cellcolor[rgb]{ .827,.827,.827}100.0& 72.2&74.7\\
&GRA&66.8& 76.1&  74.3&71.4& \cellcolor[rgb]{ .827,.827,.827}99.6& 80.3& 80.4\\
&MUMODIG& 73.2& 78.6&  77.2 &75.0& \cellcolor[rgb]{ .827,.827,.827}99.4& 83.7& 82.8\\
&PGN& 69.8& 79.4& 77.4 &75.4&  \cellcolor[rgb]{ .827,.827,.827}99.6& 84.2& 85.2\\
&MEF&  75.0& 81.1& 82.8&78.1&\cellcolor[rgb]{ .827,.827,.827}99.3&88.2& 88.1\\
&\bf{MDCS-MEF}& 74.8&84.1& 84.6& 80.1& \cellcolor[rgb]{ .827,.827,.827}99.8& 88.0& 88.0\\
&SIA&  84.2& 86.8& 86.8&  78.3&  \cellcolor[rgb]{ .827,.827,.827}99.9&  91.4&  91.5\\
&OPS& 91.2& 93.8& 93.9 & 94.4& \cellcolor[rgb]{ .827,.827,.827}99.3& 95.9& 96.0\\
&\bf{MDCS-OPS}& \bf{93.9}&\bf{95.9}&\bf{96.0}&\bf{96.6}& \cellcolor[rgb]{ .827,.827,.827}99.9& \bf{97.7}& \bf{98.3}\\
\hline
\end{tabular}
\end{table*}

\begin{table}\small
\captionsetup{font={normalsize}}
\caption{\normalsize{The white-box adversarial attacks of Pure and OCR VQA tasks in VLMs.}}
\label{tab:whitebox-vqa}
\centering
\begin{tabular}{c|c|cccc}%
\toprule %[2pt]$
Model&Type& Attack& Pre& Post$_{N}$ & ASR \\	
\midrule %[2pt]
\multirow{10}{*}{LLaVA-1.5}&\multirow{5}{*}{Pure}& FGSM& 47.3& 37.2 & 10.1\\
& & PGD& 47.3& 22.1 & 25.2\\
&& MEF& 47.3&22.1& 25.2\\
&&MI&47.3&20.9 & \textbf{26.4}\\
&&\bf{MDCS-MI}&47.3& 20.9 & \textbf{26.4}\\
\cline{2-6}
&\multirow{5}{*}{OCR}& FGSM& 58.5& 53.7 & 4.8\\
&& PGD& 58.5& 37.5 & 21.0\\
&&MEF&58.5&36.3 & 22.2\\
 & & MI& 58.5&37.2 & 21.3\\
&&\bf{MDCS-MI}&58.5&35.4 & \textbf{23.1}\\
\hline
\multirow{10}{*}{PrismVLM
}&\multirow{5}{*}{Pure}& FGSM& 56.9& 41.3&15.6\\
&& PGD& 56.9& 26.1&30.8\\
&& MEF& 56.9&25.8&31.1\\
&&MI&56.9&26.0&30.9\\
&&\bf{MDCS-MI}&56.9&22.8&\textbf{34.1}\\
\cline{2-6}
&\multirow{5}{*}{OCR}& FGSM& 61.9& 49.7&12.2\\
&& PGD& 61.9& 31.4&30.5\\
&& MEF& 61.9&31.1&30.8\\
&&MI&61.9&31.6&30.3\\
&&\bf{MDCS-MI}&61.9&27.0&\textbf{34.9}\\
\hline
\end{tabular}
\end{table}

\subsection{VQA Tasks}

VQA serves as a practical application of large multimodal models, in which the model is tasked with generating an open-ended answer from a given image and an associated question. In contrast to adversarial attacks on image classification, white-box attacks for VQA have not achieved satisfactory performance \cite{cui2024robustness, yin2024vqattack}. Therefore, this subsection specifically investigates white-box attacks on VQA, utilizing image perturbation techniques.

We conduct experiments on the TextVQA dataset \cite{singh2019towards}, targeting two representative VLMs: LLaVA-1.5 \cite{llava} and PrismVLM \cite{Prism}. To validate the effect of our MDCS, we focus on attacking visual encoders to generate AEs for LLaVA-1.5 and PrismVLM, respectively. The white-box attack success rates are reported for two input types: Pure (raw textual questions) and OCR (augmented prompts that incorporate both raw questions and OCR-extracted text from the images). 

The notations Pre, Post$_{N}$ in Tab. \ref{tab:whitebox-vqa} refer to accuracy (\%) for pre-attack, post-attack under normal setting, respectively. ASR represents the Attack Success Rate, which is derived from the values of Pre and Post$_{N}$. Tab. \ref{tab:whitebox-vqa} presents the attack success rates of our method compared to several baselines. MDCS-MI consistently achieves the highest rates across all experimental settings. The performance gain is particularly significant on the PrismVLM model, achieving an improvement of 3.0\% and 4.1\% in the Pure and OCR settings, respectively. Fig. \ref{fig:vqa_vis} presents a side-by-side comparison of the images and PrismVLM's answers on the textVQA dataset before and after being subjected to the MDCS-MI attack.

\subsection{Cross-Modal Retrieval Tasks}
\label{app: retrieval}

In the main text, we have reported R@1 attack success rates for cross-model transferability experiments for conciseness. In this subsection, we present the complete evaluation results for both image-text and text-image retrieval tasks on the Flickr30K (Table \ref{tab:addition-vqa}) and MSCOCO (Table \ref{tab:addition-vqa2}) datasets. These results include attack success rates at R@1, R@5, and R@10. 

These results validate the general applicability of our MDCS module. Integrating MDCS with a spectrum of attack frameworks, from foundational methods like SGA and DRA to the state-of-the-art SA-AET, yields consistent performance gains. In each tested cross-model transferable scenario, MDCS-SGA, MDCS-DRA, and MDCS-SAAET demonstrably outperform their baseline counterparts. This confirms that MDCS functions as a potent and model-agnostic component for amplifying the transferability of adversarial attacks across various VLP architectures and metrics.

\begin{table*}[htbp]\small
\captionsetup{font={normalsize}}
\caption{\normalsize{Detailed comparison of attack success rates on Flickr30K dataset. The gray-shaded cells represent attacks in a white-box setting, while the remaining cells show results for black-box transfer attacks. We report the attack success rates (\%) for cross modal retrieval tasks using R@1, R@5, and R@10 metrics. The adversarial perturbations are constrained by an $L_{\infty}$ norm of 8/255, generated over 10 iterations with a step-size of 2/255.}}

%The black-box adversarial attacks against multimodal models. The source model column shows VLP models we use to generate multimodal adversarial examples. R@1 denotes the attack success rate at the top-1 rank. The adversarial perturbation is set to 8/255.

\label{tab:addition-vqa}
\centering
\resizebox{.86\textwidth}{!}{
    \begin{tabular}{c|c|ccc|ccc|ccc|ccc}
    \toprule
    \multicolumn{14}{c}{\textbf{FLICKR30K (Image-Text Retrieval)}} \\
    \midrule
    \multirow{2}[4]{*}{Source} & \multirow{2}[4]{*}{Attack} & \multicolumn{3}{c}{ALBEF} & \multicolumn{3}{c}{TCL} & \multicolumn{3}{c}{CLIP$_{\text{ViT}}$} & \multicolumn{3}{c}{CLIP$_{\text{CNN}}$} \\
\cmidrule{3-14}      &   & R@1 & R@5 & R@10 & R@1 & R@5 & R@10 & R@1 & R@5 & R@10 & R@1 & R@5 & \multicolumn{1}{c}{R@10} \\
    \midrule
    \multirow{7}[2]{*}{ALBEF} & Co-Attack & \cellcolor[rgb]{ .851,  .851,  .851}97.1  & \cellcolor[rgb]{ .851,  .851,  .851}94.59  & \cellcolor[rgb]{ .851,  .851,  .851}92.60  & 39.52  & 20.40  & 14.53  & 29.82  & 11.73  & 6.61  & 31.29  & 11.73  & 5.77  \\
      & SGA & \cellcolor[rgb]{ .851,  .851,  .851}99.79  & \cellcolor[rgb]{ .851,  .851,  .851}99.80  & \cellcolor[rgb]{ .851,  .851,  .851}99.80  & 87.67  & 77.09  & 70.54  & 38.04  & 19.11  & 13.11  & 41.63  & 21.56  & 14.62  \\
      & \textbf{MDCS-SGA} & \cellcolor[rgb]{ .851,  .851,  .851}100.0  & \cellcolor[rgb]{ .851,  .851,  .851}100.0  & \cellcolor[rgb]{ .851,  .851,  .851}100.0  & 91.78  & 82.11  & 76.75  & 41.35  & 22.64  & 15.55  & 45.08  & 25.90  & 17.71  \\
      & DRA & \cellcolor[rgb]{ .851,  .851,  .851}99.79  & \cellcolor[rgb]{ .851,  .851,  .851}99.80  & \cellcolor[rgb]{ .851,  .851,  .851}99.80  & 89.78  & 80.90  & 75.75  & 46.63  & 24.30  & 18.60  & 50.32  & 27.91  & 20.19  \\
      & \textbf{MDCS-DRA} & \cellcolor[rgb]{ .851,  .851,  .851}99.9  & \cellcolor[rgb]{ .851,  .851,  .851}99.80  & \cellcolor[rgb]{ .851,  .851,  .851}99.80  & 93.26  & 85.63  & 79.26  & 49.94  & 28.35  & 20.43  & 55.56  & 32.66  & 23.38  \\
      & SA-AET & \cellcolor[rgb]{ .851,  .851,  .851}99.9  & \cellcolor[rgb]{ .851,  .851,  .851}99.80  & \cellcolor[rgb]{ .851,  .851,  .851}99.80  & 96.31  & 92.56  & 89.58  & 54.23  & 33.13  & 24.90  & 58.88  & 37.53  & 27.50  \\
      & \textbf{MDCS-SAAET} & \cellcolor[rgb]{ .851,  .851,  .851}99.9  & \cellcolor[rgb]{ .851,  .851,  .851}99.90  & \cellcolor[rgb]{ .851,  .851,  .851}99.90  & \textbf{96.52 } & \textbf{94.47 } & \textbf{91.88 } & \textbf{60.25 } & \textbf{38.53 } & \textbf{29.17 } & \textbf{60.54 } & \textbf{42.71 } & \textbf{32.13 } \\
    \midrule
    \multirow{7}[2]{*}{TCL} & Co-Attack & 49.84  & 27.45  & 60.36  & \cellcolor[rgb]{ .851,  .851,  .851}91.68  & \cellcolor[rgb]{ .851,  .851,  .851}85.23  & \cellcolor[rgb]{ .851,  .851,  .851}80.96  & 32.64  & 13.40  & 6.81  & 32.06  & 14.27  & 8.14  \\
      & SGA & 93.33  & 87.17  & 83.70  & \cellcolor[rgb]{ .851,  .851,  .851}100.0  & \cellcolor[rgb]{ .851,  .851,  .851}100.0  & \cellcolor[rgb]{ .851,  .851,  .851}100.0  & 37.42  & 18.28  & 12.20  & 42.02  & 22.73  & 15.65  \\
      & \textbf{MDCS-SGA} & 95.10  & 90.08  & 87.10  & \cellcolor[rgb]{ .851,  .851,  .851}100.0  & \cellcolor[rgb]{ .851,  .851,  .851}100.0  & \cellcolor[rgb]{ .851,  .851,  .851}100.0  & 42.45  & 21.18  & 15.75  & 46.87  & 26.74  & 20.29  \\
      & DRA & 95.31  & 91.18  & 89.80  & \cellcolor[rgb]{ .851,  .851,  .851}100.0  & \cellcolor[rgb]{ .851,  .851,  .851}100.0  & \cellcolor[rgb]{ .851,  .851,  .851}99.90  & 46.26  & 26.38  & 18.50  & 50.70  & 30.66  & 21.83  \\
      & \textbf{MDCS-DRA} & 96.66  & 92.48  & 90.70  & \cellcolor[rgb]{ .851,  .851,  .851}100.0  & \cellcolor[rgb]{ .851,  .851,  .851}100.0  & \cellcolor[rgb]{ .851,  .851,  .851}100.0  & 50.92  & 28.35  & 20.83  & 56.19  & 34.88  & 25.44  \\
      & SA-AET & 98.85  & 96.69  & 95.40  & \cellcolor[rgb]{ .851,  .851,  .851}100.0  & \cellcolor[rgb]{ .851,  .851,  .851}100.0  & \cellcolor[rgb]{ .851,  .851,  .851}100.0  & 53.99  & 33.02  & 26.42  & 59.64  & 39.96  & 30.38  \\
      & \textbf{MDCS-SAAET} & \textbf{99.17 } & \textbf{96.69 } & \textbf{96.10 } & \cellcolor[rgb]{ .851,  .851,  .851}100.0  & \cellcolor[rgb]{ .851,  .851,  .851}100.0  & \cellcolor[rgb]{ .851,  .851,  .851}100.0  & \textbf{57.79 } & \textbf{39.25 } & \textbf{29.17 } & \textbf{62.71 } & \textbf{44.19 } & \textbf{34.29 } \\
    \midrule
    \multirow{7}[2]{*}{CLIP$_{\text{ViT}}$} & Co-Attack & 8.55  & 1.50  & 0.50  & 10.01  & 2.01  & 0.70  & \cellcolor[rgb]{ .851,  .851,  .851}78.5  & \cellcolor[rgb]{ .851,  .851,  .851}57.4  & \cellcolor[rgb]{ .851,  .851,  .851}45.53  & 29.50  & 11.42  & 6.08  \\
      & SGA & 21.58  & 8.02  & 5.10  & 24.66  & 9.45  & 4.91  & \cellcolor[rgb]{ .851,  .851,  .851}100.0  & \cellcolor[rgb]{ .851,  .851,  .851}100.0  & \cellcolor[rgb]{ .851,  .851,  .851}99.9  & 52.49  & 34.04  & 25.23  \\
      & \textbf{MDCS-SGA} & 29.30  & 13.23  & 8.60  & 30.56  & 13.07  & 8.52  & \cellcolor[rgb]{ .851,  .851,  .851}100.0  & \cellcolor[rgb]{ .851,  .851,  .851}100.0  & \cellcolor[rgb]{ .851,  .851,  .851}100.0  & 57.60  & 37.32  & 28.73  \\
      & DRA & 27.95  & 11.92  & 7.90  & 29.08  & 12.56  & 7.72  & \cellcolor[rgb]{ .851,  .851,  .851}100.0  & \cellcolor[rgb]{ .851,  .851,  .851}99.90  & \cellcolor[rgb]{ .851,  .851,  .851}99.80  & 62.45  & 42.81  & 31.72  \\
      & \textbf{MDCS-DRA} & 34.41  & 15.63  & 11.40  & 35.51  & 15.08  & 9.92  & \cellcolor[rgb]{ .851,  .851,  .851}100.0  & \cellcolor[rgb]{ .851,  .851,  .851}100.0  & \cellcolor[rgb]{ .851,  .851,  .851}100.0  & 68.45  & 46.51  & 37.08  \\
      & SA-AET & 35.97  & 19.74  & 15.30  & 37.93  & 19.90  & 13.83  & \cellcolor[rgb]{ .851,  .851,  .851}100.0  & \cellcolor[rgb]{ .851,  .851,  .851}100.0  & \cellcolor[rgb]{ .851,  .851,  .851}100.0  & 69.09  & 50.11  & 42.12  \\
      & \textbf{MDCS-SAAET} & \textbf{42.34 } & \textbf{23.35 } & \textbf{18.30 } & \textbf{43.52 } & \textbf{25.03 } & \textbf{17.84 } & \cellcolor[rgb]{ .851,  .851,  .851}100.0  & \cellcolor[rgb]{ .851,  .851,  .851}100.0  & \cellcolor[rgb]{ .851,  .851,  .851}100.0  & \textbf{73.18 } & \textbf{55.29 } & \textbf{46.34 } \\
    \midrule
    \multirow{7}[2]{*}{CLIP$_{\text{CNN}}$} & Co-Attack & 10.53  & 1.60  & 0.40  & 12.54  & 2.01  & 0.70  & 27.24  & 12.05  & 6.50  & \cellcolor[rgb]{ .851,  .851,  .851}95.91  & \cellcolor[rgb]{ .851,  .851,  .851}89.75  & \cellcolor[rgb]{ .851,  .851,  .851}85.99  \\
      & SGA & 15.02  & 4.71  & 2.50  & 18.34  & 6.33  & 3.21  & 39.51  & 19.83  & 12.70  & \cellcolor[rgb]{ .851,  .851,  .851}100.0  & \cellcolor[rgb]{ .851,  .851,  .851}99.58  & \cellcolor[rgb]{ .851,  .851,  .851}99.38  \\
      & \textbf{MDCS-SGA} & 19.19  & 8.62  & 5.00  & 23.92  & 9.25  & 6.01  & 45.89  & 24.30  & 17.68  & \cellcolor[rgb]{ .851,  .851,  .851}99.87  & \cellcolor[rgb]{ .851,  .851,  .851}99.68  & \cellcolor[rgb]{ .851,  .851,  .851}99.59  \\
      & DRA & 19.08  & 6.41  & 3.30  & 22.02  & 7.14  & 3.01  & 48.34  & 26.48  & 17.99  & \cellcolor[rgb]{ .851,  .851,  .851}100.0  & \cellcolor[rgb]{ .851,  .851,  .851}99.58  & \cellcolor[rgb]{ .851,  .851,  .851}99.28  \\
      & \textbf{MDCS-DRA} & 23.04  & 10.02  & 6.40  & 25.82  & 10.45  & 6.51  & 55.83  & 31.98  & 22.15  & \cellcolor[rgb]{ .851,  .851,  .851}100.0  & \cellcolor[rgb]{ .851,  .851,  .851}99.68  & \cellcolor[rgb]{ .851,  .851,  .851}99.59  \\
      & SA-AET & 23.88  & 9.12  & 6.00  & 25.18  & 10.05  & 6.91  & 54.60  & 31.26  & 23.78  & \cellcolor[rgb]{ .851,  .851,  .851}100.0  & \cellcolor[rgb]{ .851,  .851,  .851}100.0  & \cellcolor[rgb]{ .851,  .851,  .851}99.79  \\
      & \textbf{MDCS-SAAET} & \textbf{30.76 } & \textbf{13.93 } & \textbf{10.00 } & \textbf{33.40 } & \textbf{15.68 } & \textbf{10.52 } & \textbf{61.72 } & \textbf{41.33 } & \textbf{31.50 } & \cellcolor[rgb]{ .851,  .851,  .851}100.0  & \cellcolor[rgb]{ .851,  .851,  .851}100.0  & \cellcolor[rgb]{ .851,  .851,  .851}100.0  \\
    \midrule
    \multicolumn{14}{c}{\textbf{FLICKR30K (Text-Image Retrieval)}} \\
    \midrule
    \multirow{2}[4]{*}{Source} & \multirow{2}[4]{*}{Attack} & \multicolumn{3}{c}{ALBEF} & \multicolumn{3}{c}{TCL} & \multicolumn{3}{c}{CLIP$_{\text{ViT}}$} & \multicolumn{3}{c}{CLIP$_{\text{CNN}}$} \\
\cmidrule{3-14}      &   & R@1 & R@5 & R@10 & R@1 & R@5 & R@10 & R@1 & R@5 & R@10 & R@1 & R@5 & R@10 \\
    \midrule
    \multirow{7}[2]{*}{ALBEF} & Co-Attack & \cellcolor[rgb]{ .851,  .851,  .851}98.36  & \cellcolor[rgb]{ .851,  .851,  .851}96.41  & \cellcolor[rgb]{ .851,  .851,  .851}94.86  & 51.24  & 31.90  & 24.41  & 38.92  & 23.31  & 17.01  & 41.99  & 25.18  & 18.55  \\
      & SGA & \cellcolor[rgb]{ .851,  .851,  .851}99.95  & \cellcolor[rgb]{ .851,  .851,  .851}99.88  & \cellcolor[rgb]{ .851,  .851,  .851}99.88  & 87.88  & 77.07  & 70.94  & 46.17  & 28.40  & 21.96  & 50.36  & 32.24  & 24.67  \\
      & \textbf{MDCS-SGA} & \cellcolor[rgb]{ .851,  .851,  .851}99.98  & \cellcolor[rgb]{ .851,  .851,  .851}99.96  & \cellcolor[rgb]{ .851,  .851,  .851}99.96  & 91.24  & 81.87  & 75.55  & 49.71  & 30.88  & 23.68  & 53.93  & 35.15  & 27.69  \\
      & DRA & \cellcolor[rgb]{ .851,  .851,  .851}99.91  & \cellcolor[rgb]{ .851,  .851,  .851}99.92  & \cellcolor[rgb]{ .851,  .851,  .851}99.86  & 90.52  & 82.05  & 76.42  & 57.28  & 38.15  & 29.87  & 59.11  & 41.41  & 33.59  \\
      & \textbf{MDCS-DRA} & \cellcolor[rgb]{ .851,  .851,  .851}99.98  & \cellcolor[rgb]{ .851,  .851,  .851}99.92  & \cellcolor[rgb]{ .851,  .851,  .851}99.92  & 92.98  & 85.67  & 80.63  & 59.31  & 40.74  & 31.95  & 62.44  & 45.20  & 36.05  \\
      & SA-AET & \cellcolor[rgb]{ .851,  .851,  .851}100.0  & \cellcolor[rgb]{ .851,  .851,  .851}99.90  & \cellcolor[rgb]{ .851,  .851,  .851}99.90  & 96.19  & 92.21  & 89.32  & 63.50  & 45.95  & 36.81  & 65.18  & 48.37  & 39.87  \\
      & \textbf{MDCS-SAAET} & \cellcolor[rgb]{ .851,  .851,  .851}100.0  & \cellcolor[rgb]{ .851,  .851,  .851}99.94  & \cellcolor[rgb]{ .851,  .851,  .851}99.92  & \textbf{96.71 } & \textbf{93.24 } & \textbf{90.48 } & \textbf{67.01 } & \textbf{49.19 } & \textbf{40.86 } & \textbf{67.89 } & \textbf{51.80 } & \textbf{43.51 } \\
    \midrule
    \multirow{7}[2]{*}{TCL} & Co-Attack & 60.36  & 41.59  & 33.26  & \cellcolor[rgb]{ .851,  .851,  .851}95.48  & \cellcolor[rgb]{ .851,  .851,  .851}90.32  & \cellcolor[rgb]{ .851,  .851,  .851}86.74  & 42.69  & 26.44  & 20.37  & 47.82  & 30.47  & 23.13  \\
      & SGA & 92.84  & 87.31  & 83.46  & \cellcolor[rgb]{ .851,  .851,  .851}100.0  & \cellcolor[rgb]{ .851,  .851,  .851}100.0  & \cellcolor[rgb]{ .851,  .851,  .851}100.0  & 46.39  & 29.57  & 22.50  & 51.36  & 33.38  & 25.23  \\
      & \textbf{MDCS-SGA} & 94.69  & 89.54  & 86.78  & \cellcolor[rgb]{ .851,  .851,  .851}100.0  & \cellcolor[rgb]{ .851,  .851,  .851}100.0  & \cellcolor[rgb]{ .851,  .851,  .851}100.0  & 49.19  & 31.58  & 25.05  & 55.71  & 37.17  & 28.33  \\
      & DRA & 95.35  & 91.57  & 89.04  & \cellcolor[rgb]{ .851,  .851,  .851}100.0  & \cellcolor[rgb]{ .851,  .851,  .851}99.98  & \cellcolor[rgb]{ .851,  .851,  .851}99.94  & 56.80  & 39.31  & 31.62  & 61.54  & 43.57  & 35.01  \\
      & \textbf{MDCS-DRA} & 96.42  & 92.92  & 90.96  & \cellcolor[rgb]{ .851,  .851,  .851}100.0  & \cellcolor[rgb]{ .851,  .851,  .851}100.0  & \cellcolor[rgb]{ .851,  .851,  .851}100.0  & 59.76  & 41.91  & 35.02  & 65.01  & 46.89  & 38.60  \\
      & SA-AET & 98.48  & 96.82  & 95.43  & \cellcolor[rgb]{ .851,  .851,  .851}99.98  & \cellcolor[rgb]{ .851,  .851,  .851}99.96  & \cellcolor[rgb]{ .851,  .851,  .851}99.96  & 63.27  & 46.13  & 38.42  & 68.44  & 50.68  & 42.04  \\
      & \textbf{MDCS-SAAET} & \textbf{98.88 } & \textbf{97.13 } & \textbf{96.20 } & \cellcolor[rgb]{ .851,  .851,  .851}100.0  & \cellcolor[rgb]{ .851,  .851,  .851}100.0  & \cellcolor[rgb]{ .851,  .851,  .851}99.98  & \textbf{65.98 } & \textbf{49.71 } & \textbf{42.02 } & \textbf{71.01 } & \textbf{54.20 } & \textbf{45.92 } \\
    \midrule
    \multirow{7}[2]{*}{CLIP$_{\text{ViT}}$} & Co-Attack & 20.18  & 9.54  & 7.12  & 21.29  & 9.51  & 6.78  & \cellcolor[rgb]{ .851,  .851,  .851}87.50  & \cellcolor[rgb]{ .851,  .851,  .851}77.95  & \cellcolor[rgb]{ .851,  .851,  .851}73.40  & 38.49  & 23.19  & 17.87  \\
      & SGA & 34.89  & 18.07  & 13.69  & 35.83  & 18.79  & 14.11  & \cellcolor[rgb]{ .851,  .851,  .851}100.0  & \cellcolor[rgb]{ .851,  .851,  .851}100.0  & \cellcolor[rgb]{ .851,  .851,  .851}100.0  & 60.38  & 42.58  & 35.19  \\
      & \textbf{MDCS-SGA} & 40.69  & 21.90  & 16.28  & 39.81  & 22.45  & 16.02  & \cellcolor[rgb]{ .851,  .851,  .851}100.0  & \cellcolor[rgb]{ .851,  .851,  .851}100.0  & \cellcolor[rgb]{ .851,  .851,  .851}100.0  & 66.00  & 48.20  & 39.71  \\
      & DRA & 43.29  & 24.63  & 18.82  & 44.83  & 25.76  & 19.39  & \cellcolor[rgb]{ .851,  .851,  .851}100.0  & \cellcolor[rgb]{ .851,  .851,  .851}100.0  & \cellcolor[rgb]{ .851,  .851,  .851}99.96  & 69.47  & 53.88  & 45.06  \\
      & \textbf{MDCS-DRA} & 48.15  & 28.63  & 21.69  & 48.40  & 28.65  & 22.24  & \cellcolor[rgb]{ .851,  .851,  .851}100.0  & \cellcolor[rgb]{ .851,  .851,  .851}100.0  & \cellcolor[rgb]{ .851,  .851,  .851}100.0  & 73.41  & 57.79  & 49.24  \\
      & SA-AET & 50.28  & 32.08  & 25.76  & 51.36  & 32.77  & 26.18  & \cellcolor[rgb]{ .851,  .851,  .851}100.0  & \cellcolor[rgb]{ .851,  .851,  .851}99.98  & \cellcolor[rgb]{ .851,  .851,  .851}99.98  & 74.00  & 59.22  & 51.62  \\
      & \textbf{MDCS-SAAET} & \textbf{54.86 } & \textbf{36.89 } & \textbf{29.32 } & \textbf{55.45 } & \textbf{36.49 } & \textbf{29.16 } & \cellcolor[rgb]{ .851,  .851,  .851}100.0  & \cellcolor[rgb]{ .851,  .851,  .851}99.98  & \cellcolor[rgb]{ .851,  .851,  .851}99.96  & \textbf{77.53 } & \textbf{63.15 } & \textbf{55.61 } \\
    \midrule
    \multirow{7}[2]{*}{CLIP$_{\text{CNN}}$} & Co-Attack & 23.62  & 11.40  & 8.21  & 26.05  & 12.69  & 8.79  & 40.62  & 24.71  & 18.82  & \cellcolor[rgb]{ .851,  .851,  .851}96.50  & \cellcolor[rgb]{ .851,  .851,  .851}92.75  & \cellcolor[rgb]{ .851,  .851,  .851}90.35  \\
      & SGA & 28.60  & 14.99  & 10.92  & 32.26  & 17.32  & 12.67  & 51.16  & 33.45  & 25.56  & \cellcolor[rgb]{ .851,  .851,  .851}100.0  & \cellcolor[rgb]{ .851,  .851,  .851}99.90  & \cellcolor[rgb]{ .851,  .851,  .851}99.73  \\
      & \textbf{MDCS-SGA} & 33.44  & 17.86  & 12.90  & 36.48  & 19.02  & 14.09  & 56.22  & 38.64  & 30.14  & \cellcolor[rgb]{ .851,  .851,  .851}100.0  & \cellcolor[rgb]{ .851,  .851,  .851}99.95  & \cellcolor[rgb]{ .851,  .851,  .851}99.86  \\
      & DRA & 33.96  & 18.50  & 13.49  & 37.45  & 20.74  & 15.13  & 59.02  & 40.55  & 32.75  & \cellcolor[rgb]{ .851,  .851,  .851}99.93  & \cellcolor[rgb]{ .851,  .851,  .851}99.59  & \cellcolor[rgb]{ .851,  .851,  .851}99.28  \\
      & \textbf{MDCS-DRA} & 39.43  & 22.29  & 16.60  & 41.05  & 23.53  & 17.83  & 63.89  & 45.60  & 36.89  & \cellcolor[rgb]{ .851,  .851,  .851}100.0  & \cellcolor[rgb]{ .851,  .851,  .851}99.93  & \cellcolor[rgb]{ .851,  .851,  .851}99.82  \\
      & SA-AET & 37.89  & 21.68  & 16.52  & 41.69  & 24.33  & 18.14  & 63.14  & 44.45  & 35.52  & \cellcolor[rgb]{ .851,  .851,  .851}99.97  & \cellcolor[rgb]{ .851,  .851,  .851}99.78  & \cellcolor[rgb]{ .851,  .851,  .851}99.62  \\
      & \textbf{MDCS-SAAET} & \textbf{45.02 } & \textbf{27.71 } & \textbf{21.17 } & \textbf{47.57 } & \textbf{30.10 } & \textbf{22.68 } & \textbf{69.56 } & \textbf{52.18 } & \textbf{43.76 } & \cellcolor[rgb]{ .851,  .851,  .851}99.97  & \cellcolor[rgb]{ .851,  .851,  .851}99.88  & \cellcolor[rgb]{ .851,  .851,  .851}99.82  \\
    \bottomrule
    \end{tabular}
    }
\end{table*}

\begin{table*}[htbp]\small
\captionsetup{font={normalsize}}
\caption{\normalsize{Detailed comparison of attack success rates on MSCOCO dataset. We report the attack success rates (\%) for both Image-Text Retrieval and Text-Image Retrieval tasks using R@1, R@5, and R@10 metrics. The adversarial perturbations are constrained by an $L_{\infty}$ norm of 8/255, generated over 10 iterations with a step size of 2/255.}}
\label{tab:addition-vqa2}
\centering
\resizebox{.88\textwidth}{!}{
    \begin{tabular}{c|c|ccc|ccc|ccc|ccc}
    \toprule
    \multicolumn{14}{c}{\textbf{MSCOCO (Image-Text Retrieval)}} \\
    \midrule
    \multirow{2}[4]{*}{Source} & \multirow{2}[4]{*}{Attack} & \multicolumn{3}{c}{ALBEF} & \multicolumn{3}{c}{TCL} & \multicolumn{3}{c}{CLIP$_{\text{ViT}}$} & \multicolumn{3}{c}{CLIP$_{\text{CNN}}$} \\
\cmidrule{3-14}      &   & R@1 & R@5 & R@10 & R@1 & R@5 & R@10 & R@1 & R@5 & R@10 & R@1 & R@5 & R@10 \\
    \midrule
    \multirow{7}[2]{*}{ALBEF} & Co-Attack & \cellcolor[rgb]{ .851,  .851,  .851}96.65  & \cellcolor[rgb]{ .851,  .851,  .851}94.74  & \cellcolor[rgb]{ .851,  .851,  .851}93.12  & 57.33  & 37.24  & 28.53  & 50.71  & 33.10  & 26.44  & 52.06  & 33.89  & 27.47  \\
      & SGA & \cellcolor[rgb]{ .851,  .851,  .851}99.95  & \cellcolor[rgb]{ .851,  .851,  .851}99.70  & \cellcolor[rgb]{ .851,  .851,  .851}99.65  & 87.22  & 76.39  & 69.34  & 64.06  & 46.76  & 39.02  & 63.34  & 47.67  & 39.72  \\
      & \textbf{MDCS-SGA} & \cellcolor[rgb]{ .851,  .851,  .851}100.0  & \cellcolor[rgb]{ .851,  .851,  .851}99.96  & \cellcolor[rgb]{ .851,  .851,  .851}99.84  & 88.23  & 77.94  & 70.73  & 64.94  & 49.08  & 40.44  & 66.12  & 49.63  & 41.83  \\
      & DRA & \cellcolor[rgb]{ .851,  .851,  .851}99.92  & \cellcolor[rgb]{ .851,  .851,  .851}99.68  & \cellcolor[rgb]{ .851,  .851,  .851}99.44  & 88.78  & 78.81  & 72.24  & 69.21  & 52.56  & 43.97  & 69.06  & 52.13  & 43.87  \\
      & \textbf{MDCS-DRA} & \cellcolor[rgb]{ .851,  .851,  .851}100.0  & \cellcolor[rgb]{ .851,  .851,  .851}99.85  & \cellcolor[rgb]{ .851,  .851,  .851}99.77  & 90.29  & 80.75  & 74.00  & 71.77  & 54.26  & 45.61  & 71.19  & 55.04  & 46.49  \\
      & SA-AET & \cellcolor[rgb]{ .851,  .851,  .851}100.0  & \cellcolor[rgb]{ .851,  .851,  .851}99.96  & \cellcolor[rgb]{ .851,  .851,  .851}99.96  & 96.98  & 93.45  & 90.11  & 76.61  & 61.51  & 53.76  & 75.32  & 61.03  & 53.19  \\
      & \textbf{MDCS-SAAET} & \cellcolor[rgb]{ .851,  .851,  .851}100.0  & \cellcolor[rgb]{ .851,  .851,  .851}99.96  & \cellcolor[rgb]{ .851,  .851,  .851}99.94  & \textbf{97.54 } & \textbf{94.03 } & \textbf{91.04 } & \textbf{79.28 } & \textbf{65.25 } & \textbf{57.33 } & \textbf{78.05 } & \textbf{64.16 } & \textbf{56.44 } \\
    \midrule
    \multirow{7}[2]{*}{TCL} & Co-Attack & 65.22  & 45.39  & 36.43  & \cellcolor[rgb]{ .851,  .851,  .851}94.95  & \cellcolor[rgb]{ .851,  .851,  .851}91.18  & \cellcolor[rgb]{ .851,  .851,  .851}89.08  & 55.28  & 38.04  & 29.73  & 56.68  & 37.31  & 29.82  \\
      & SGA & 92.78  & 86.95  & 83.68  & \cellcolor[rgb]{ .851,  .851,  .851}100.0  & \cellcolor[rgb]{ .851,  .851,  .851}99.98  & \cellcolor[rgb]{ .851,  .851,  .851}99.96  & 58.68  & 44.11  & 36.70  & 61.05  & 45.91  & 38.36  \\
      & \textbf{MDCS-SGA} & 93.50  & 88.63  & 85.36  & \cellcolor[rgb]{ .851,  .851,  .851}100.0  & \cellcolor[rgb]{ .851,  .851,  .851}100.0  & \cellcolor[rgb]{ .851,  .851,  .851}100.0  & 62.15  & 46.39  & 39.56  & 62.44  & 48.68  & 40.94  \\
      & DRA & 94.61  & 90.49  & 87.90  & \cellcolor[rgb]{ .851,  .851,  .851}100.0  & \cellcolor[rgb]{ .851,  .851,  .851}100.0  & \cellcolor[rgb]{ .851,  .851,  .851}99.96  & 71.08  & 54.81  & 46.53  & 70.29  & 54.94  & 45.89  \\
      & \textbf{MDCS-DRA} & 95.64  & 91.81  & 89.05  & \cellcolor[rgb]{ .851,  .851,  .851}100.0  & \cellcolor[rgb]{ .851,  .851,  .851}100.0  & \cellcolor[rgb]{ .851,  .851,  .851}99.98  & 72.30  & 56.36  & 48.76  & 72.82  & 57.71  & 48.58  \\
      & SA-AET & 97.96  & 95.78  & 93.95  & \cellcolor[rgb]{ .851,  .851,  .851}100.0  & \cellcolor[rgb]{ .851,  .851,  .851}99.98  & \cellcolor[rgb]{ .851,  .851,  .851}99.96  & 76.00  & 61.54  & 54.63  & 75.64  & 61.00  & 53.51  \\
      & \textbf{MDCS-SAAET} & \textbf{98.12 } & \textbf{95.97 } & \textbf{94.55 } & \cellcolor[rgb]{ .851,  .851,  .851}100.0  & \cellcolor[rgb]{ .851,  .851,  .851}99.96  & \cellcolor[rgb]{ .851,  .851,  .851}99.90  & \textbf{78.33 } & \textbf{64.84 } & \textbf{57.29 } & \textbf{79.12 } & \textbf{65.79 } & \textbf{56.86 } \\
    \midrule
    \multirow{7}[2]{*}{CLIP$_{\text{ViT}}$} & Co-Attack & 26.35  & 11.97  & 7.20  & 28.23  & 12.89  & 8.19  & \cellcolor[rgb]{ .851,  .851,  .851}88.78  & \cellcolor[rgb]{ .851,  .851,  .851}78.21  & \cellcolor[rgb]{ .851,  .851,  .851}70.98  & 47.36  & 31.49  & 25.29  \\
      & SGA & 43.88  & 25.53  & 17.99  & 43.73  & 25.44  & 18.23  & \cellcolor[rgb]{ .851,  .851,  .851}100.0  & \cellcolor[rgb]{ .851,  .851,  .851}100.0  & \cellcolor[rgb]{ .851,  .851,  .851}100.0  & 71.60  & 56.46  & 48.63  \\
      & \textbf{MDCS-SGA} & 48.47  & 28.96  & 20.96  & 47.75  & 28.29  & 20.71  & \cellcolor[rgb]{ .851,  .851,  .851}100.0  & \cellcolor[rgb]{ .851,  .851,  .851}100.0  & \cellcolor[rgb]{ .851,  .851,  .851}100.0  & 74.09  & 59.94  & 52.22  \\
      & DRA & 52.08  & 32.42  & 23.88  & 51.67  & 31.65  & 23.73  & \cellcolor[rgb]{ .851,  .851,  .851}100.0  & \cellcolor[rgb]{ .851,  .851,  .851}99.95  & \cellcolor[rgb]{ .851,  .851,  .851}99.95  & 80.79  & 67.72  & 60.64  \\
      & \textbf{MDCS-DRA} & 56.30  & 36.41  & 28.00  & 55.32  & 34.58  & 26.21  & \cellcolor[rgb]{ .851,  .851,  .851}100.0  & \cellcolor[rgb]{ .851,  .851,  .851}99.97  & \cellcolor[rgb]{ .851,  .851,  .851}99.98  & 82.63  & 71.04  & 63.75  \\
      & SA-AET & 57.90  & 38.93  & 29.81  & 57.41  & 38.14  & 30.06  & \cellcolor[rgb]{ .851,  .851,  .851}99.96  & \cellcolor[rgb]{ .851,  .851,  .851}99.97  & \cellcolor[rgb]{ .851,  .851,  .851}99.93  & 84.51  & 73.48  & 66.63  \\
      & \textbf{MDCS-SAAET} & \textbf{62.80 } & \textbf{44.26 } & \textbf{34.89 } & \textbf{62.43 } & \textbf{42.56 } & \textbf{34.53 } & \cellcolor[rgb]{ .851,  .851,  .851}100.0  & \cellcolor[rgb]{ .851,  .851,  .851}99.95  & \cellcolor[rgb]{ .851,  .851,  .851}99.93  & \textbf{85.70 } & \textbf{76.12 } & \textbf{70.44 } \\
    \midrule
    \multirow{7}[2]{*}{CLIP$_{\text{CNN}}$} & Co-Attack & 29.49  & 13.26  & 8.28  & 31.83  & 15.11  & 9.81  & 53.15  & 36.11  & 28.78  & \cellcolor[rgb]{ .851,  .851,  .851}97.79  & \cellcolor[rgb]{ .851,  .851,  .851}94.29  & \cellcolor[rgb]{ .851,  .851,  .851}92.26  \\
      & SGA & 36.58  & 19.14  & 11.90  & 38.47  & 20.38  & 14.19  & 62.69  & 47.46  & 38.64  & \cellcolor[rgb]{ .851,  .851,  .851}99.96  & \cellcolor[rgb]{ .851,  .851,  .851}99.84  & \cellcolor[rgb]{ .851,  .851,  .851}99.68  \\
      & \textbf{MDCS-SGA} & 40.53  & 21.90  & 14.64  & 43.28  & 24.11  & 17.05  & 66.58  & 52.15  & 44.23  & \cellcolor[rgb]{ .851,  .851,  .851}99.96  & \cellcolor[rgb]{ .851,  .851,  .851}99.97  & \cellcolor[rgb]{ .851,  .851,  .851}99.95  \\
      & DRA & 41.09  & 21.58  & 14.41  & 43.25  & 23.59  & 16.36  & 70.89  & 55.21  & 46.46  & \cellcolor[rgb]{ .851,  .851,  .851}99.75  & \cellcolor[rgb]{ .851,  .851,  .851}99.70  & \cellcolor[rgb]{ .851,  .851,  .851}99.32  \\
      & \textbf{MDCS-DRA} & 45.60  & 25.67  & 18.05  & 48.17  & 27.71  & 20.31  & 75.70  & 61.02  & 52.57  & \cellcolor[rgb]{ .851,  .851,  .851}99.96  & \cellcolor[rgb]{ .851,  .851,  .851}99.92  & \cellcolor[rgb]{ .851,  .851,  .851}99.78  \\
      & SA-AET & 43.57  & 24.85  & 17.35  & 47.04  & 26.65  & 19.24  & 72.61  & 57.95  & 50.20  & \cellcolor[rgb]{ .851,  .851,  .851}99.92  & \cellcolor[rgb]{ .851,  .851,  .851}99.84  & \cellcolor[rgb]{ .851,  .851,  .851}99.76  \\
      & \textbf{MDCS-SAAET} & \textbf{50.71 } & \textbf{31.17 } & \textbf{22.87 } & \textbf{53.04 } & \textbf{33.03 } & \textbf{24.24 } & \textbf{79.13 } & \textbf{66.90 } & \textbf{59.04 } & \cellcolor[rgb]{ .851,  .851,  .851}100.0  & \cellcolor[rgb]{ .851,  .851,  .851}99.95  & \cellcolor[rgb]{ .851,  .851,  .851}99.90  \\
    \midrule
    \multicolumn{14}{c}{\textbf{MSCOCO (Text-Image Retrieval)}} \\
    \midrule
    \multirow{2}[4]{*}{Source} & \multirow{2}[4]{*}{Attack} & \multicolumn{3}{c}{ALBEF} & \multicolumn{3}{c}{TCL} & \multicolumn{3}{c}{CLIP$_{\text{ViT}}$} & \multicolumn{3}{c}{CLIP$_{\text{CNN}}$} \\
\cmidrule{3-14}      &   & R@1 & R@5 & R@10 & R@1 & R@5 & R@10 & R@1 & R@5 & R@10 & R@1 & R@5 & R@10 \\
    \midrule
    \multirow{7}[2]{*}{ALBEF} & Co-Attack & \cellcolor[rgb]{ .851,  .851,  .851}98.33  & \cellcolor[rgb]{ .851,  .851,  .851}96.60  & \cellcolor[rgb]{ .851,  .851,  .851}95.30  & 64.19  & 46.17  & 37.83  & 57.36  & 42.19  & 35.53  & 60.74  & 45.90  & 38.77  \\
      & SGA & \cellcolor[rgb]{ .851,  .851,  .851}99.95  & \cellcolor[rgb]{ .851,  .851,  .851}99.81  & \cellcolor[rgb]{ .851,  .851,  .851}99.73  & 87.96  & 76.98  & 70.35  & 70.01  & 55.43  & 48.06  & 70.95  & 56.31  & 49.00  \\
      & \textbf{MDCS-SGA} & \cellcolor[rgb]{ .851,  .851,  .851}99.99  & \cellcolor[rgb]{ .851,  .851,  .851}99.95  & \cellcolor[rgb]{ .851,  .851,  .851}99.92  & 88.28  & 77.93  & 71.47  & 70.88  & 56.85  & 49.08  & 72.08  & 57.69  & 50.76  \\
      & DRA & \cellcolor[rgb]{ .851,  .851,  .851}99.96  & \cellcolor[rgb]{ .851,  .851,  .851}99.79  & \cellcolor[rgb]{ .851,  .851,  .851}99.66  & 90.01  & 80.25  & 73.80  & 74.89  & 61.34  & 53.78  & 75.11  & 61.74  & 54.39  \\
      & \textbf{MDCS-DRA} & \cellcolor[rgb]{ .851,  .851,  .851}99.97  & \cellcolor[rgb]{ .851,  .851,  .851}99.89  & \cellcolor[rgb]{ .851,  .851,  .851}99.82  & 90.92  & 81.56  & 75.38  & 76.28  & 62.76  & 55.42  & 77.42  & 63.53  & 56.37  \\
      & SA-AET & \cellcolor[rgb]{ .851,  .851,  .851}99.99  & \cellcolor[rgb]{ .851,  .851,  .851}99.96  & \cellcolor[rgb]{ .851,  .851,  .851}99.92  & 96.87  & 93.16  & 90.34  & 79.99  & 68.16  & 61.23  & 80.86  & 68.58  & 61.39  \\
      & \textbf{MDCS-SAAET} & \cellcolor[rgb]{ .851,  .851,  .851}99.99  & \cellcolor[rgb]{ .851,  .851,  .851}99.96  & \cellcolor[rgb]{ .851,  .851,  .851}99.93  & \textbf{97.25 } & \textbf{94.02 } & \textbf{91.46 } & \textbf{81.91 } & \textbf{70.70 } & \textbf{63.88 } & \textbf{82.51 } & \textbf{70.60 } & \textbf{63.92 } \\
    \midrule
    \multirow{7}[2]{*}{TCL} & Co-Attack & 72.41  & 56.37  & 48.16  & \cellcolor[rgb]{ .851,  .851,  .851}97.87  & \cellcolor[rgb]{ .851,  .851,  .851}95.32  & \cellcolor[rgb]{ .851,  .851,  .851}93.42  & 62.33  & 46.90  & 39.68  & 66.45  & 49.95  & 42.72  \\
      & SGA & 93.07  & 87.82  & 84.17  & \cellcolor[rgb]{ .851,  .851,  .851}100.0  & \cellcolor[rgb]{ .851,  .851,  .851}100.0  & \cellcolor[rgb]{ .851,  .851,  .851}100.0  & 64.89  & 50.93  & 43.85  & 67.46  & 53.56  & 46.29  \\
      & \textbf{MDCS-SGA} & 94.29  & 89.04  & 85.66  & \cellcolor[rgb]{ .851,  .851,  .851}100.0  & \cellcolor[rgb]{ .851,  .851,  .851}100.0  & \cellcolor[rgb]{ .851,  .851,  .851}100.0  & 67.56  & 52.76  & 45.77  & 69.72  & 56.45  & 49.05  \\
      & DRA & 95.74  & 91.87  & 89.31  & \cellcolor[rgb]{ .851,  .851,  .851}99.99  & \cellcolor[rgb]{ .851,  .851,  .851}99.99  & \cellcolor[rgb]{ .851,  .851,  .851}99.98  & 75.14  & 61.13  & 53.92  & 76.87  & 63.69  & 56.22  \\
      & \textbf{MDCS-DRA} & 96.49  & 92.67  & 90.14  & \cellcolor[rgb]{ .851,  .851,  .851}100.0  & \cellcolor[rgb]{ .851,  .851,  .851}99.99  & \cellcolor[rgb]{ .851,  .851,  .851}99.99  & 76.68  & 63.53  & 56.24  & 78.68  & 65.78  & 58.75  \\
      & SA-AET & 98.06  & 95.93  & 94.57  & \cellcolor[rgb]{ .851,  .851,  .851}99.99  & \cellcolor[rgb]{ .851,  .851,  .851}99.95  & \cellcolor[rgb]{ .851,  .851,  .851}99.94  & 79.96  & 67.85  & 61.25  & 81.09  & 69.43  & 62.51  \\
      & \textbf{MDCS-SAAET} & \textbf{98.17 } & \textbf{96.16 } & \textbf{94.94 } & \cellcolor[rgb]{ .851,  .851,  .851}99.99  & \cellcolor[rgb]{ .851,  .851,  .851}99.97  & \cellcolor[rgb]{ .851,  .851,  .851}99.93  & \textbf{81.55 } & \textbf{70.18 } & \textbf{63.87 } & \textbf{83.07 } & \textbf{71.90 } & \textbf{65.28 } \\
    \midrule
    \multirow{7}[2]{*}{CLIP$_{\text{ViT}}$} & Co-Attack & 36.69  & 22.86  & 17.71  & 38.42  & 23.51  & 17.88  & \cellcolor[rgb]{ .851,  .851,  .851}96.72  & \cellcolor[rgb]{ .851,  .851,  .851}91.28  & \cellcolor[rgb]{ .851,  .851,  .851}85.46  & 58.45  & 43.78  & 36.77  \\
      & SGA & 51.04  & 33.62  & 27.06  & 50.83  & 34.42  & 27.92  & \cellcolor[rgb]{ .851,  .851,  .851}100.0  & \cellcolor[rgb]{ .851,  .851,  .851}100.0  & \cellcolor[rgb]{ .851,  .851,  .851}100.0  & 75.59  & 63.21  & 56.28  \\
      & \textbf{MDCS-SGA} & 54.76  & 36.70  & 29.42  & 53.25  & 36.42  & 29.66  & \cellcolor[rgb]{ .851,  .851,  .851}100.0  & \cellcolor[rgb]{ .851,  .851,  .851}100.0  & \cellcolor[rgb]{ .851,  .851,  .851}100.0  & 78.42  & 66.56  & 59.98  \\
      & DRA & 61.51  & 43.47  & 36.22  & 60.81  & 43.89  & 36.16  & \cellcolor[rgb]{ .851,  .851,  .851}100.0  & \cellcolor[rgb]{ .851,  .851,  .851}100.0  & \cellcolor[rgb]{ .851,  .851,  .851}99.99  & 84.61  & 74.06  & 67.92  \\
      & \textbf{MDCS-DRA} & 64.52  & 46.76  & 38.89  & 62.95  & 46.17  & 38.31  & \cellcolor[rgb]{ .851,  .851,  .851}100.0  & \cellcolor[rgb]{ .851,  .851,  .851}100.0  & \cellcolor[rgb]{ .851,  .851,  .851}99.99  & 86.00  & 77.26  & 71.28  \\
      & SA-AET & 66.21  & 49.54  & 42.24  & 65.50  & 49.36  & 41.87  & \cellcolor[rgb]{ .851,  .851,  .851}99.99  & \cellcolor[rgb]{ .851,  .851,  .851}99.97  & \cellcolor[rgb]{ .851,  .851,  .851}99.95  & 87.32  & 77.99  & 72.27  \\
      & \textbf{MDCS-SAAET} & \textbf{70.09 } & \textbf{53.98 } & \textbf{46.53 } & \textbf{68.75 } & \textbf{52.97 } & \textbf{45.26 } & \cellcolor[rgb]{ .851,  .851,  .851}100.0  & \cellcolor[rgb]{ .851,  .851,  .851}99.98  & \cellcolor[rgb]{ .851,  .851,  .851}99.98  & \textbf{89.18 } & \textbf{80.13 } & \textbf{74.96 } \\
    \midrule
    \multirow{7}[2]{*}{CLIP$_{\text{CNN}}$} & Co-Attack & 41.50  & 26.14  & 20.51  & 43.44  & 27.92  & 21.61  & 60.15  & 45.53  & 38.56  & \cellcolor[rgb]{ .851,  .851,  .851}98.54  & \cellcolor[rgb]{ .851,  .851,  .851}96.16  & \cellcolor[rgb]{ .851,  .851,  .851}94.71  \\
      & SGA & 46.29  & 30.00  & 23.70  & 48.75  & 32.86  & 26.01  & 67.89  & 53.23  & 46.93  & \cellcolor[rgb]{ .851,  .851,  .851}99.95  & \cellcolor[rgb]{ .851,  .851,  .851}99.85  & \cellcolor[rgb]{ .851,  .851,  .851}99.76  \\
      & \textbf{MDCS-SGA} & 50.04  & 32.89  & 26.19  & 51.72  & 35.26  & 27.89  & 72.24  & 58.48  & 51.84  & \cellcolor[rgb]{ .851,  .851,  .851}99.97  & \cellcolor[rgb]{ .851,  .851,  .851}99.96  & \cellcolor[rgb]{ .851,  .851,  .851}99.94  \\
      & DRA & 52.32  & 35.71  & 28.85  & 54.26  & 37.89  & 30.60  & 74.83  & 62.53  & 55.45  & \cellcolor[rgb]{ .851,  .851,  .851}99.93  & \cellcolor[rgb]{ .851,  .851,  .851}99.76  & \cellcolor[rgb]{ .851,  .851,  .851}99.64  \\
      & \textbf{MDCS-DRA} & 56.64  & 39.42  & 32.21  & 57.84  & 41.30  & 33.90  & 78.85  & 66.94  & 60.37  & \cellcolor[rgb]{ .851,  .851,  .851}99.97  & \cellcolor[rgb]{ .851,  .851,  .851}99.91  & \cellcolor[rgb]{ .851,  .851,  .851}99.84  \\
      & SA-AET & 54.86  & 38.13  & 31.21  & 57.84  & 40.16  & 32.96  & 76.95  & 64.59  & 57.86  & \cellcolor[rgb]{ .851,  .851,  .851}99.91  & \cellcolor[rgb]{ .851,  .851,  .851}99.81  & \cellcolor[rgb]{ .851,  .851,  .851}99.65  \\
      & \textbf{MDCS-SAAET} & \textbf{60.76 } & \textbf{44.01 } & \textbf{36.78 } & \textbf{62.13 } & \textbf{45.37 } & \textbf{37.92 } & \textbf{82.75 } & \textbf{71.77 } & \textbf{65.13 } & \cellcolor[rgb]{ .851,  .851,  .851}99.97  & \cellcolor[rgb]{ .851,  .851,  .851}99.91  & \cellcolor[rgb]{ .851,  .851,  .851}99.80  \\
    \bottomrule
    \end{tabular}%
    }
\end{table*}


\end{document}